%% file: manuscript_v1.tex
\documentclass[11pt]{article}

\usepackage[margin=1in]{geometry}
\usepackage[T1]{fontenc}
\usepackage[utf8]{inputenc}
\usepackage{lmodern}
\usepackage{microtype}
\usepackage{amsmath,amssymb,amsthm,mathtools,bm,mathrsfs}
\usepackage{booktabs,array,enumitem}
\usepackage{longtable}
\usepackage{graphicx}
\usepackage{float}
\usepackage{xcolor}
\usepackage[round,authoryear]{natbib}
\usepackage[hidelinks,hypertexnames=false]{hyperref}
\usepackage[nameinlink,noabbrev]{cleveref}
\hypersetup{
  pdftitle={Honest Physical-Support Inference after Latent Dictionary Learning: Collision Singularities and Minimax Resolution},
  pdfauthor={Guan-Ju Peng},
  pdfsubject={Theory-first arXiv manuscript on honest sparse-support inference with a latent learned dictionary},
  pdfkeywords={dictionary learning, sparse representation, uncertainty quantification, confidence sets, minimax inference}
}

\newtheorem{theorem}{Theorem}[section]
\newtheorem{proposition}[theorem]{Proposition}
\newtheorem{lemma}[theorem]{Lemma}
\newtheorem{corollary}[theorem]{Corollary}
\theoremstyle{definition}

\newtheorem{assumption}[theorem]{Assumption}
\theoremstyle{remark}

\newcommand{\R}{\mathbb R}
\newcommand{\Pp}{\mathbb P}
\newcommand{\E}{\mathbb E}
\newcommand{\1}{\mathbf 1}
\newcommand{\KL}{\operatorname{KL}}
\newcommand{\diam}{\operatorname{diam}}
\newcommand{\tr}{\operatorname{tr}}

\newcommand{\pr}{\mathrm{pr}}
\newcommand{\prof}{\mathrm{prof}}

\newcommand{\tightlist}{%
  \setlength{\itemsep}{0pt}\setlength{\parskip}{0pt}}
\allowdisplaybreaks

\title{Honest Physical-Support Inference after Latent Dictionary Learning:\\
Collision Singularities and Minimax Resolution}
\author{Guan-Ju Peng\thanks{Institute of Data Science and Information Computing,
National Chung Hsing University. Email:
\href{mailto:gjpeng@email.nchu.edu.tw}{\texttt{gjpeng@email.nchu.edu.tw}}.
ORCID: \href{https://orcid.org/0000-0001-5508-9485}{0000-0001-5508-9485}.}}
\date{July 2026}

\begin{document}
\maketitle

\begin{abstract}
\input{sections/00_abstract}
\end{abstract}

\input{sections/01_introduction}
\input{sections/02_collision}

\input{sections/02_related_work}
\input{sections/03_problem}
\input{sections/04_method}
\input{sections/05_main_results}
\input{sections/06_proof_ideas}

\input{sections/08_discussion}

\clearpage
\appendix
\section{Complete proofs of the headline results}
\label{app:main-proofs}
\input{appendices/A_main_proofs}

\clearpage
\part*{Technical appendices}
\input{appendices/B_technical_details}

\bibliographystyle{plainnat}
\bibliography{references}

\end{document}

%% file: sections/00_abstract.tex
Sparse-support uncertainty is usually quantified after treating the
dictionary as known.  This conditioning can be misleading when the
dictionary was itself learned from latent sparse mixtures: near a collision
of coherent atoms, a new signal may reliably identify the active physical
group while the training experiment cannot identify which physical rays
inside that group generated it.  Freezing one learned dictionary converts
this unresolved calibration uncertainty into apparently precise but
label-dependent support statements.

We formulate inference for the active \emph{physical rays}---unit atoms
modulo sign---after latent dictionary learning.  In a fixed-dimensional
Gaussian train--test experiment, the block membership, hierarchy, and
collision shell are supplied, while the child rays, block shape, residual
orientation, and separated anchor are unknown.  We retain every dictionary compatible with
a robust training-moment region, profile the sparse test representation over
all retained dictionaries, and project the surviving configurations onto a
permutation-invariant physical-support space.  The output is therefore
allowed to be cross-sheet inconclusive, group-resolved but child-ambiguous,
or fine-support resolved, instead of being forced to name unstable atom
labels.

The analysis identifies the statistical price and the decision-theoretic
gain of this construction.  Residual block orientation first enters the
latent training density at cubic order, so its information is of order
\(s^6\), where \(s\) is the within-block collision scale.  The confidence
correspondence has high-probability-over-training conditional test coverage,
and its resolution is governed by three separate information budgets:
parent detectability, test-time support separation, and learned-dictionary
orientation.  In the resolved fixed-shell regime its projective Hausdorff
diameter contracts at the minimax-optimal rate
\(s\wedge(\sqrt N\,s^2)^{-1}\), up to constants.  A restricted task theorem
further shows when coefficient asymmetry lets test replication supplement
the training information and when no amount of additional test data can
remove the calibration bottleneck.  These results turn learned-dictionary
uncertainty into an honest, resolution-adaptive sparse statement and an
explicit guide for allocating training versus test measurements.

%% file: sections/01_introduction.tex
\section{Introduction}
\label{sec:introduction}

Sparse representations are useful only when their active atoms support a
scientifically meaningful statement.  In source separation, material
decomposition, and learned feature analysis, the support is interpreted as
the set of components present in a new signal.  Most support-inference
methods nevertheless condition on the dictionary.  That is appropriate
when the atoms are known physical templates.  It is a different statistical
problem when the dictionary has first been estimated from a separate sample
of mixtures whose supports and coefficients were never observed.

The distinction becomes consequential near a collision of coherent atoms.
Suppose that the physical rays inside a coherent block are separated at
scale \(s\), the dictionary-training sample has size \(N\), and the new
signal is observed \(T\) times.  If the dictionary were known, the test
experiment can distinguish child supports once its separation information
\begin{equation}
I_S=\frac{T\beta_-^2s^2}{\sigma_+^2}
\label{eq:intro-support-information}
\end{equation}
is large.  After latent dictionary learning, however, residual block
orientation is learned only through information of order
\begin{equation}
I_D=Ns^6.
\label{eq:intro-dictionary-information}
\end{equation}
There is therefore a nonempty regime \(I_S\gg1\) but \(I_D\lesssim1\).
In that regime, a known-dictionary analysis can distinguish coordinate
supports, while the training experiment does not identify the physical rays
to which those coordinates refer.  Freezing one learned dictionary hides
this mismatch: it can turn unresolved calibration uncertainty into a
precise but alignment-dependent atom label.  The honest conclusion is
coarser---the coherent block is active, but its active physical children are
not yet resolved.

This paper asks for the finest sparse statement justified by the \emph{joint}
train--test experiment.  Given unlabeled training data \(Y^N\) and an
independent replicated test sample \(Z^T\), we construct a random compact
correspondence \(\widehat{\mathfrak C}\) satisfying
\begin{equation}
\Pp_{\rm train}\!\left[
  \Pp_{\rm test}\!\left\{
    \vartheta_\star\in\widehat{\mathfrak C}\mid Y^N
  \right\}\ge1-\alpha_T
\right]\ge1-\alpha_D .
\label{eq:intro-conditional-coverage}
\end{equation}
Here \(\vartheta_\star\) is a marked set of active \emph{physical rays},
modulo simultaneous dictionary--support relabelling; it is not a coordinate
support attached to one arbitrary alignment.  The primary output is the
collection of all physical support statements compatible with both data
sources.  A lower/upper support pair is available as a summary, but would
discard which alternative branches coexist and is therefore not our
estimand.

Sparsity is not merely a regularizer in this problem.  It creates the
information that identifies the atoms.  At the centered collision core,
full activation of the child block preserves a continuous rotational
ambiguity.  In contrast, mixtures generated from proper subdictionaries
leave an orientation signature.  We prove that the observed training
density cancels its orientation dependence through second order and first
reveals it at cubic order.  The cubic coefficient is nonzero exactly when
the exchangeable code law assigns positive power to a proper nonempty
subdictionary.  Squaring the score produces the \(Ns^6\) orientation gate
in \cref{eq:intro-dictionary-information}.

Our procedure keeps that singular training uncertainty explicit.  First, a
simultaneous confidence region for observable training moments is inverted
into the set of all compatible dictionary orbits.  Second, every retained
dictionary, support, and positive coefficient vector is profiled against one
confidence ball for the replicated test mean.  Third, the surviving
dictionary--support pairs are mapped to physical-ray sets.  The truth is
retained on one joint event, so this construction needs neither a
support-labelled calibration sample nor a multiplicity correction over all
candidate supports.

The statistical gain is adaptive \emph{resolution}, not universal
narrowness.  The same output has three scientifically distinct operating
states.  It is cross-sheet inconclusive when even the coherent block and a
separated alternative cannot be distinguished; it identifies the coherent
parent while retaining several child explanations when the block is
detectable but the fine support is not; and it contracts to one transported
physical support branch once both the test and dictionary gates open.  The
corresponding information budgets are
\begin{equation}
I_G^{(r)}=\frac{Tr^2\beta_-^2}{\sigma_+^2},
\qquad I_S=\frac{T\beta_-^2s^2}{\sigma_+^2},
\qquad I_D=Ns^6.
\label{eq:intro-three-gates}
\end{equation}
They have a direct measurement-allocation interpretation.  A closed parent
gate calls for more test information to detect the correct sheet; a closed
support gate calls for more test replication or higher signal-to-noise
ratio; and a closed dictionary gate generally calls for more informative
dictionary training.  A separate task theorem identifies the exceptional
directions in which coefficient asymmetry lets the test experiment break a
training-side symmetry.

\paragraph{Contributions.}
Under the fixed-dimensional, fixed-shell experiment stated in
\cref{sec:problem}, we establish four claims.
\begin{enumerate}[leftmargin=*,label=\textbf{C\arabic*.}]
\item We derive a quotient-complete statistical geometry for latent sparse
training near atom collision.  The all-pair Hellinger/Kullback--Leibler
chord is regular in observable invariants, but its physical inverse loses
powers of \(s\); nuisance-profiled orientation information is of order
\(s^6\).

\item We construct a finite-sample cross-dictionary confidence correspondence with
the two-level honesty guarantee \cref{eq:intro-conditional-coverage}.  It
does not select or align a point dictionary before test inference and
requires no observed training codes.

\item We prove a three-gate diameter bound and matching fixed-shell minimax
lower bounds.  Once the parent and support gates are open, the projective
Hausdorff diameter is, up to constants,
\[
s\wedge\frac{1}{\sqrt N\,s^2},
\]
so the method attains the finest physical resolution permitted by the
joint experiment.

\item We characterize task-dependent symmetry breaking.  For arbitrary
fixed \(2\le r\le q-1\) we give the local coefficient-profiled test
information, and on a separately declared one-dimensional orbit we obtain
a finite matching train--test rate and an explicit two-atom amplitude
contrast crossover.
\end{enumerate}

Each ingredient has close predecessors: group inference under coherence,
honest sparse confidence sets, set-valued prediction, dictionary
identifiability, learned-estimator uncertainty quantification, and
invariant-based orbit recovery.  We do not claim priority over any one of
them.  The contribution is their conjunction in a latent train--test sparse
inverse problem, together with a sharp physical-support diameter through a
collision-singular learned representation.

The remainder of the paper begins with a concrete collision thought
experiment and a closest-work reduction, then defines the statistical target
and confidence construction.  The main theorems separate the obstruction,
the method, its optimality, and the downstream task effect.  A proof roadmap
exposes the mechanism, and the appendices give a complete proof chain,
including the finite-sample moment construction, quotient geometry,
measurability argument, and least-favourable pairs.

%% file: sections/02_collision.tex
\section{A collision thought experiment}
\label{sec:thought-experiment}

Before introducing the full affine-simplex model, consider four coherent
children \(d_1,\ldots,d_4\) around one parent ray and a test signal using two
of them.  Two dictionary orientations \(D\) and \(D'\) can induce nearly the
same distribution of latent sparse training mixtures even though the
physical child rays have moved by order \(s\).  A point learner must choose
one orientation.  Conditional on that choice, a high-SNR test signal may
strongly prefer the coordinate pair \(\{1,2\}\).  But under another
training-compatible orientation, the same coordinate pair denotes a
different set of physical rays.

This produces a gap between two questions:
\begin{center}
\begin{tabular}{p{0.25\linewidth}p{0.31\linewidth}p{0.31\linewidth}}
\toprule
& Known or frozen dictionary & Learned dictionary treated honestly \\
\midrule
Target & coordinate support in one supplied frame
       & active physical rays across all training-compatible frames \\
Uncertainty source & test noise
       & test noise and latent dictionary-training uncertainty \\
Fine answer requires & \(I_S\) large
       & \(I_S\) and \(I_D\) large \\
When only \(I_S\) is large & leaf support
       & coherent parent, children unresolved \\
\bottomrule
\end{tabular}
\end{center}

The distinction is not semantic.  The known-dictionary output can be
correct as a statement about the selected coordinate system while being too
precise as a statement about the physical components.  Our inferential
object therefore does not force a leaf label.  It returns every marked
physical support still compatible with training and test data.

The three possible outputs are summarized by
\begin{equation}
\underbrace{\text{coherent or separated?}}_{I_G^{(r)}}
\quad\longrightarrow\quad
\underbrace{\text{which child support?}}_{I_S}
\quad\longrightarrow\quad
\underbrace{\text{which physical child rays?}}_{I_D}.
\label{eq:decision-sequence}
\end{equation}
If the first gate is closed, the result is cross-sheet inconclusive.  If the
first is open but either of the last two is closed, the result can certify
the coherent parent without pretending to know its active children in
physical space.  If all three gates are open, the same construction resolves
one transported support branch.  The main results show both that this
sequence is attainable and that no uniformly honest method can bypass its
closed gates.

There is a second lesson.  More test data do not automatically compensate
for inadequate dictionary training.  Along a compensated orientation path,
the coefficients may change so that the entire test mean stays fixed while
the active rays move.  Test replication then contains no information about
that direction.  With asymmetric coefficients, however, some orientation
directions change the test mean transversely and the test sample can help.
\Cref{thm:task-symmetry} makes this distinction exact.

%% file: sections/02_related_work.tex
\section{Related work and problem boundary}
\label{sec:related}

The ingredients of our construction have substantial precedents.  Group-level
inference under collinearity, honest sparse confidence sets, set-valued
prediction, dictionary identifiability, learned-estimator uncertainty
quantification, and estimation modulo symmetries are all established topics.
The relevant boundary is therefore not whether any one of these tools has
appeared before.  It is whether they answer the following joint question: after
a dictionary has been learned from noisy mixtures with latent sparse codes,
what physical support statement about a separate fixed sparse signal is honest
as child atoms approach collision?  We organize the comparison around the
changes in design, target, and guarantee that this question requires.

\subsection{Known-design support inference versus a latent-trained design}

For an observed regression design, a coherent group can remain inferentially
detectable even when its individual variables cannot.  Group-bound provides
confidence statements for groups without requiring favorable design
conditions, while hierarchical testing moves from coarse to fine groups of
correlated variables with error control
\citep{Meinshausen2015,MandozziBuhlmann2016}.  Setwise variable-selection
procedures provide a further modern version of this principle
\citep{OrganKenneyGu2026}.  Thus neither ``infer the group when the members are
ambiguous'' nor traversal of a supplied hierarchy is itself new.

Honesty and adaptation also have a mature known-design theory.  Honest sparse
regression confidence sets require separation from statistically
indistinguishable alternatives \citep{NicklVandeGeer2013}.  Model confidence bounds and
confidence sets of models report multiple plausible supports rather than
forcing a single selected model \citep{LiEtAl2019,LewisBattey2025}.  These works
establish the general lesson that support uncertainty should be set-valued and
that adaptation has an information-theoretic price.

Our design is not observed.  It is inferred from an independent sample whose
codes and supports are latent, and an atom index has physical meaning only
together with the dictionary that carries it.  Near collision, two training-
compatible dictionaries can attach the same index to different rays.  The
resulting object is therefore a joint confidence correspondence of marked
dictionary--support configurations, projected onto physical ray sets, rather
than a confidence set for coefficients or model indices under one fixed design.
The block membership, hierarchy, and collision shell are supplied; the child
rays, block shape, and residual orientation are learned from latent mixtures.
What adapts is the finest physical support resolution justified jointly by
training and test information.  In particular, the additional $Ns^6$
orientation gate in our
upper and lower bounds has no analogue in a problem that conditions on the
design.

\subsection{Dictionary recovery versus uncertainty in a downstream sparse
statement}

Dictionary-learning theory supplies conditions for local identifiability of a
generating dictionary, including Bernoulli--Gaussian and exact-sparse models
\citep{GribonvalSchnass2010,WuYu2018}.  It also gives localization guarantees in
the presence of noise and outliers and minimax limits for dictionary
estimation \citep{GribonvalJenattonBach2015,JungEldarGoertz2016}.  Perturbation
analyses quantify the stability of sparse codes and downstream predictors when
the supplied dictionary changes \citep{MehtaGray2013}.  These results explain
when a dictionary or a code can be recovered; they do not turn uncertainty
from the blind training experiment into an honest statement about the active
rays of one new signal.

That distinction matters at a collision.  A small global dictionary error is
not by itself an adequate downstream certificate: the same perturbation can be
irrelevant to detecting the coherent parent yet decisive for naming its active
children.  Conversely, a deterministic coding-stability bound typically
presupposes a margin that is precisely what degenerates along the collision
shell.  Our estimand and loss are consequently support-aware and
permutation-invariant.  We do not claim a new general dictionary-identification
algorithm or a sharper global dictionary-estimation rate.  The contribution is
to propagate the specific nonregular training geometry through a downstream
physical-support decision and to match the resulting fixed-shell resolution by
a lower bound for that same decision.

\subsection{Learned-estimator UQ versus uncertainty in the forward dictionary}

Rigorous uncertainty quantification is available for learned sparse solvers
and for much broader learned predictors.  Confidence intervals have been
constructed for LISTA with a supplied forward operator
\citep{HoppeEtAl2023}, and a subsequent non-asymptotic framework calibrates a
large class of learned estimators using appropriate estimation and calibration
data \citep{HoppeEtAl2024}.  Hence neither UQ for LISTA nor generic UQ for a
learned map is a priority claim of this paper.

Here the uncertain learned object is the forward dictionary itself.  The
training data contain neither the latent dictionary codes nor labelled
test-support calibration pairs, and the target changes physical resolution as
atoms collide.  Residual block orientation is first visible in the latent
training law through a cubic invariant, yielding order-$s^6$ information rather
than a uniform root-$N$ variance correction.  We therefore invert the training
experiment and profile the test representation over every retained dictionary;
we do not freeze a learned operator and calibrate only the conditional error of
a solver.  The current paper analyzes exact compact profiling, not an unrolled
network, a learned stopping rule, or a scalable computational improvement.

\subsection{Set-valued and latent-label prediction versus a fixed multi-atom
physical support}

Set-valued prediction already provides coverage while allowing the output to
grow on ambiguous instances.  This includes least-ambiguous classification
sets, conformal sets with adaptive size, and finite-sample multiclass or
multilabel confidence sets \citep{SadinleLeiWasserman2019,RomanoSesiaCandes2020,
CauchoisGuptaDuchi2021}.  Confidence sets for cluster trees quantify uncertainty
in a learned hierarchy \citep{KimEtAl2016}.  More directly, conformal clustering
handles permutation-ambiguous latent cluster labels, and weighted conformal
clustering studies the additional distributional mismatch involved in moving
from synthetic to latent labels \citep{NathHurAllenCluster2026,NathHurAllen2026}.
These precedents rule out broad claims based merely on latent labels,
permutation alignment, hierarchical outputs, or singleton-to-set adaptation.

Our target is instead the active set of physical rays for one fixed
multi-atom inverse problem.  It is not a random categorical label, and no
sample of true support labels is available for conformal calibration.  The
guarantee is two-level: with high probability over blind dictionary training,
the conditional test-time coverage of the physical-support correspondence is
high.  Its size is controlled in projective Hausdorff distance and is compared
with a minimax lower bound along the same collision shell.  Thus set-valued
output is the reporting language, while the train--test collision geometry
determines which set sizes are statistically unavoidable.

\subsection{Orbit and singular geometry versus an honest train--test target}

Estimating an orbit rather than an arbitrary labelled representative is a
standard response to group symmetries, and low-degree separating invariants can
govern sample complexity under group actions \citep{BandeiraEtAl2023}.  Cubic
invariants and sixth-order information are therefore not new phenomena in
isolation.  Singular mixture models and weak-identification theory likewise
show how degenerating information invalidates ordinary regular asymptotics
\citep{HoNguyen2019,Kaji2021}.  In a regular square-ICA setting, contraction and
local asymptotic normality have also been formulated modulo source
ambiguities \citep{DattaGhoshPolson2025}.  Finally, projecting or profiling a
confidence region to obtain a set-valued estimand has general precedents in
inference for identified sets \citep{ChernozhukovHongTamer2007}.

Our use of this geometry is narrower and downstream-facing.  The child
permutation quotient is carried through a collision-local dictionary experiment
to a marked physical support; the same cubic obstruction then appears in both
the achievable diameter and a least-favorable support construction.  A separate
task-aligned theorem specifies when test coefficient asymmetry breaks that
training symmetry and when it does not.  The paper-level claim is therefore the
following conjunction, not ownership of any component:
\begin{equation*}
\begin{split}
&\text{latent exchangeable dictionary training}
+\text{ collision-singular quotient geometry}\\
&\quad+\text{ training-conditional physical-support profiling}\\
&\quad+\text{ matching fixed-shell resolution limits}.
\end{split}
\end{equation*}
To our knowledge, the existing results above do not establish this exact
conjunction.  This is a bounded closest-work statement, not an absolute
priority claim.  In particular, we do not claim the first learned-dictionary
UQ method, the first group-valued support set, or the first sixth-order orbit
law.

%% file: sections/03_problem.tex
\section{Problem formulation and physical target}
\label{sec:problem}

\subsection{Why the plug-in target is not stable}

The usual two-stage pipeline learns a dictionary \(\widehat D\) and then
solves a known-dictionary sparse-coding problem.  Conditional uncertainty
quantification for the second stage can be statistically valid for the
optimization problem defined by \(\widehat D\), yet fail to answer the
physical question.  If two admissible training dictionaries differ by a
permutation or a small rotation of a coherent child block, the same atom
index refers to different rays.  A narrow interval or a singleton support at
one selected alignment may therefore express confidence about a coordinate
system that the training data did not identify.

A two-ray thought experiment makes the issue visible.  For normalized atoms
with \(d_1^\top d_2=\rho\), distinguishing the two one-atom signals
\(\beta d_1\) and \(\beta d_2\) depends on
\(\beta^2(1-\rho)/\sigma^2\), whereas detecting that their common coherent
group is active depends primarily on \(\beta^2/\sigma^2\).  Hence there is a
regime in which the group is detectable but the child identity is not.  Our
formal model uses a higher-dimensional block because the learned dictionary
also has an orientation degree of freedom that is invisible to low-order
training information.  The inferential target must accommodate both
phenomena: test-time ambiguity among children and training-time uncertainty
about the physical child rays themselves.

\subsection{Independent training and test experiments}

Fix an ambient dimension \(q\ge4\), let \(n=q+1\), and fix a test support
cardinality \(2\le r\le q-1\).  The training sample consists of independent
latent sparse mixtures
\begin{equation}
Y_i=D_\star c_i+\xi_i,
\qquad
\xi_i\sim\mathcal N_q(0,\nu_\star I_q),
\qquad i=1,\ldots,N .
\tag{2.4}
\label{eq:training-model}
\end{equation}
Neither the support nor the nonzero entries of \(c_i\) are observed.  The
flagship unknown-law branch is Bernoulli--Gaussian,
\begin{equation}
c_{ij}=I_{ij}G_{ij},\qquad
I_{ij}\stackrel{\mathrm{iid}}{\sim}\operatorname{Bernoulli}(p_\star),
\qquad
G_{ij}\stackrel{\mathrm{iid}}{\sim}\mathcal N(0,1),
\tag{2.7}
\label{eq:bg-code}
\end{equation}
with \(p_\star\in[p_-,p_+]\Subset(0,1/2)\) and
\(\nu_\star\in[\nu_-,\nu_+]\Subset(0,\infty)\).  We also treat a separately
typed fixed, known exchangeable masked-Gaussian code law; the required
nondegeneracy condition is stated in \cref{sec:results}.

Independently of training, the test experiment observes replicated
measurements of one fixed sparse signal,
\begin{equation}
Z_\ell\stackrel{\mathrm{iid}}{\sim}
\mathcal N_q(\mu_\star,\sigma_\star^2I_q),
\qquad \ell=1,\ldots,T,\qquad
\sigma_\star\in[\sigma_-,\sigma_+] .
\tag{2.8}
\label{eq:test-model}
\end{equation}
On the coherent, or fine-support, sheet,
\begin{equation}
\mu_\star
=\sum_{j\in S_\star}x_{\star j}d_{\star j},
\qquad
S_\star\in\binom{[q]}r,\qquad
x_{\star j}\in[\beta_-,\beta_+],\quad \beta_->0 .
\tag{2.9}
\label{eq:fine-mean}
\end{equation}
The competing separated sheet is
\begin{equation}
\mu_\star=\gamma_\star a_\star,
\qquad
\gamma_\star\in\Gamma_A=[\gamma_-,\gamma_+]\Subset(0,\infty),
\tag{2.10}
\label{eq:anchor-mean}
\end{equation}
where \(a_\star\) is a separated anchor atom.  The symbol \(\partial\)
denotes this sheet; it is not a no-signal null.

\subsection{A supplied coherent block near collision}

Let \(\R^q=U\oplus\operatorname{span}\{u\}\), with \(\dim U=q-1\), and
choose regular-simplex reference vertices \(v_1,\ldots,v_q\in U\) satisfying
\begin{equation}
\sum_{j=1}^qv_j=0,\qquad
v_i^\top v_j=-\frac1{q-1}\ (i\ne j),\qquad
\sum_{j=1}^qv_jv_j^\top=\frac q{q-1}I_U .
\label{eq:simplex-reference}
\end{equation}
For \(z\in U\), define the unit atom
\begin{equation}
d(z)=\sqrt{1-\|z\|^2}\,u+z .
\label{eq:atom-chart}
\end{equation}
The coherent children are
\begin{equation}
z_j=b+Lv_j,\qquad d_j=d(z_j),\qquad j=1,\ldots,q .
\label{eq:affine-simplex}
\end{equation}
For one externally supplied collision scale \(s\in(0,s_0]\), the compact
shell obeys
\begin{equation}
\max_j\|b+Lv_j\|\le C_0s,\qquad
\kappa_-s\le\sigma_{\min}(L)\le\sigma_{\max}(L)\le\kappa_+s .
\tag{2.1}
\label{eq:shell}
\end{equation}
The anchor projector and training-noise variance obey
\begin{equation}
\|P_a-P_{a_0}\|_F\le C_as,
\qquad 0<\nu_-\le\nu\le\nu_+<\infty.
\tag{2.2}
\label{eq:anchor-noise-shell}
\end{equation}
Canonical positive-cap signs remove local
sign ambiguity; the remaining dictionary nonidentifiability is simultaneous
permutation of the \(q\) children.

\begin{lemma}[Affine-simplex chart]
\label{lem:affine-simplex}
Let \(\widetilde z_1,\ldots,\widetilde z_q\in U\) be affinely independent and
put \(b=q^{-1}\sum_j\widetilde z_j\).  There is a unique invertible linear
map \(L:U\to U\) such that
\(\widetilde z_j=b+Lv_j\) for every \(j\).  If
\(\|b\|\le C_bs\), the centered diameter is of order \(s\), and the
normalized affine condition number is bounded, then \cref{eq:shell} holds
after changing fixed shell constants.
\end{lemma}

\begin{proof}
Let \(w_j=\widetilde z_j-b\).  Then \(\sum_jw_j=0\).  The simplex
vertices span \(U\), and their only linear relation is
\(\sum_jv_j=0\).  Therefore the assignment \(v_j\mapsto w_j\) extends to
a unique linear map \(L:U\to U\): it is well defined because the sole
relation is preserved, and it is unique because the \(v_j\) span.
Affine independence of the \(\widetilde z_j\) implies that the \(w_j\)
span \(U\), so \(L\) is invertible.  Finally, the tight-frame identity for
the \(v_j\) makes the centered diameter comparable to
\(\|L\|_{\mathrm{op}}\).  If that diameter is comparable to \(s\) and
the condition number of \(L\) is bounded, then all singular values are
bounded above and below by fixed multiples of \(s\).  Together with
\(\|b\|\le C_bs\) and the fixed unit norms of the \(v_j\), this also gives
\(\max_j\|b+Lv_j\|\le C_0s\), proving \cref{eq:shell} after changing the
fixed constants.
\end{proof}

Thus the observed child block need not be a regular simplex.  The regular
simplex supplies quotient-complete coordinates for every locally affinely
independent collision in the declared shell.

\subsection{The estimand: marked sets of physical rays}

Let \([D]\) denote the dictionary orbit under child permutation, and let
\(\mathcal Q^D_{q,s}\) be the resulting compact quotient shell.  Its physical
transport metric is
\begin{equation}
d_{\mathcal Q}^{D}([D],[D'])
=\min_{\pi\in\mathfrak S_q}
\max\left\{
\max_{1\le j\le q}\|z_j-z'_{\pi(j)}\|,
\|P_a-P_a'\|_F
\right\}.
\tag{2.10b}
\label{eq:physical-dictionary-metric}
\end{equation}
For unit
vectors define the projective metric
\begin{equation}
d_{\pr}([v],[w])=\sqrt{1-(v^\top w)^2},
\label{eq:projective-metric}
\end{equation}
with induced Hausdorff metric \(d_H^{\pr}\).  Write \(\mathcal K(E)\) for
the nonempty compact subsets of a metric space \(E\).  The physical target is
\begin{equation}
\vartheta(D,S)=\bigl(1,\{[d_j]:j\in S\}\bigr),\qquad
\vartheta(D,\partial)=\bigl(0,\{[a]\}\bigr),
\tag{2.11}
\label{eq:physical-target}
\end{equation}
taking values in
\begin{equation}
\mathfrak V_q=\{0,1\}_{\mathrm{disc}}
\times\mathcal K(\mathbb{RP}^{q-1}) .
\label{eq:target-space}
\end{equation}
We equip it with the genuine compact metric
\begin{equation}
d_{\mathfrak V}\{(m,A),(m',A')\}
=|m-m'|+d_H^{\pr}(A,A').
\tag{2.11a}
\label{eq:marked-target-metric}
\end{equation}
The binary mark distinguishes the coherent and separated sheets.  Physical
survival of mark \(1\) alone is the formal parent/group certificate; multiple
surviving ray sets with that mark record unresolved child identity.  Physical
resolution is measured by
\begin{equation}
\diam_{\pr}(C)
=\sup_{(m,A),(m',A')\in C}d_H^{\pr}(A,A'),
\tag{2.11b}
\label{eq:physical-diameter}
\end{equation}
which intentionally omits the mark: sheet certification and within-sheet
ray resolution are different questions.

The target in \cref{eq:physical-target} changes the decision being made.  We
do not ask the procedure to guess an arbitrary child alignment.  We ask it
to retain every physical support statement still compatible with both
experiments.  This permits honest coarsening when the data do not justify a
leaf-level answer and exact physical localization when they do.

\subsection{Standing scope and least-favourable hard core}

The following assumption collects the fixed-shell scope used by the main
results and records the balanced submodel used only to prove the lower
bounds.  It is deliberately narrower than a general coherent dictionary
model.

\begin{assumption}[Fixed-shell train--test experiment]
\label{ass:frozen-shell}
All dimensions, compact parameter bounds, confidence levels, and shell
constants in \cref{eq:training-model,eq:test-model,eq:shell} are fixed.
The code law is either one fixed exchangeable masked-Gaussian law satisfying
\begin{equation}
a_1\ge a_->0,
\qquad \Delta_{\mathcal L}\ge\Delta_->0,
\tag{2.6}
\label{eq:law-nondegeneracy}
\end{equation}
with the quantities defined in
\cref{eq:exchangeable-coefficients,eq:delta-law}, or the
Bernoulli--Gaussian law \cref{eq:bg-code}.  Test coefficients belong to the
positive box in \cref{eq:fine-mean}; test noise lies in
\([\sigma_-,\sigma_+]\Subset(0,\infty)\).

For the least-favourable submodels, fix a unit \(e_1\in U\) and a constant
\begin{equation}
\lambda_\star\in(\kappa_-,\kappa_+)\cap(0,C_0).
\tag{2.3}
\label{eq:balanced-lambda}
\end{equation}
Then take
\begin{equation}
a_0=(u+e_1)/\sqrt2,
\qquad b=0,
\qquad L=\lambda_\star sR,
\qquad P_a=P_{a_0},
\tag{2.3a}
\label{eq:balanced-hard-core}
\end{equation}
The interval \(\Gamma_A\) contains, with a fixed interior margin, the anchor
projection of every fine mean in this balanced submodel:
\begin{equation}
\left\{
a_0^\top\sum_{j\in S}x_jd_j:
S\in\tbinom{[q]}r,\ x_j\in[\beta_-,\beta_+],\
b=0,\ L=\lambda_\star sR
\right\}
\subset\operatorname{int}(\Gamma_A).
\tag{2.10a}
\label{eq:anchor-amplitude-transversality}
\end{equation}
The supplied
\(s_0\) is chosen once, before observing data, small enough for the local
chart, quotient inverse, support matching, and separation margins below.
No adaptation over an unknown collision scale is asserted.
\end{assumption}

Let \(\mathfrak P^p_{q,r}(s)\) denote the resulting joint unknown-\(p\)
parameter class, including both the fine and separated sheets.  The fixed
known-law class is defined analogously after omitting \(p\).  Unless stated
otherwise, all constants below may depend on the fixed quantities in
\cref{ass:frozen-shell}, but not on \(N,T\), or the supplied \(s\).

%% file: sections/04_method.tex
\section{An honest cross-dictionary confidence construction}
\label{sec:method}

The construction is a fixed-dimensional statistical benchmark.  It follows
the inferential question rather than the order in which the proof was
discovered:
\begin{equation}
\begin{aligned}
Y^N
&\longrightarrow
\text{training-compatible dictionary region}\\
&\longrightarrow
\text{joint dictionary--support profiles}
\longrightarrow
\text{physical-support correspondence}.
\end{aligned}
\label{eq:method-roadmap}
\end{equation}
The decisive departure from plug-in sparse coding is the middle step.  No
single \(\widehat D\) is selected and frozen.  Dictionary uncertainty is
kept as a joint nuisance until after the test support and coefficients have
been profiled.

\subsection{Observable training invariants}

Write \(P_j=d_jd_j^\top\) and
\begin{equation}
\Sigma_2(D)=\sum_{j=1}^nP_j,\qquad
\Sigma_4(D)=\sum_{j=1}^nP_j^{\odot2}.
\label{eq:projector-powers}
\end{equation}
Here \(\odot\) is normalized full symmetrization, so a centered Gaussian
vector with covariance \(C\) satisfies
\(\E X^{\otimes4}=3C^{\odot2}\).  Define
\begin{equation}
M_2=\E(YY^\top),
\qquad
K_4=\E(Y^{\otimes4})-3M_2^{\odot2}.
\label{eq:training-moments}
\end{equation}
Under Bernoulli--Gaussian coding,
\begin{equation}
M_2=\nu I_q+p\Sigma_2(D),
\qquad
K_4=3p(1-p)\Sigma_4(D).
\tag{3.3}
\label{eq:bg-moments}
\end{equation}
Because every atom has unit norm, a normalized double contraction gives
\begin{equation}
\frac{\operatorname{Tr}_2K_4}{3n}=p(1-p).
\tag{3.4}
\label{eq:p-trace}
\end{equation}
The branch \(p<1/2\) makes \(p\) and subsequently \(\nu\) identifiable.
The finite-sample procedure does \emph{not} plug sample moments into this
inverse; it retains \((p,\nu,D)\) jointly.

For a fixed known exchangeable code law, let \(K\) be the number of active
atoms and \(\lambda_K\) the conditional active variance.  Put
\begin{equation}
a_1=\frac{\E(\lambda_KK)}n,\quad
a_{2d}=\frac{\E(\lambda_K^2K)}n,\quad
a_{2o}=\frac{\E\{\lambda_K^2K(K-1)\}}{n(n-1)},
\label{eq:exchangeable-coefficients}
\end{equation}
and
\begin{equation}
\Delta_{\mathcal L}=a_{2d}-a_{2o}
=\frac{\E\{\lambda_K^2K(n-K)\}}{n(n-1)}.
\tag{2.5}
\label{eq:delta-law}
\end{equation}
Then
\begin{align}
M_2&=\nu I_q+a_1\Sigma_2(D),
\tag{3.1}
\label{eq:knownlaw-m2}\\
K_4&=3\left[
\Delta_{\mathcal L}\Sigma_4(D)
+(a_{2o}-a_1^2)\Sigma_2(D)^{\odot2}
\right].
\tag{3.2}
\label{eq:knownlaw-k4}
\end{align}
The quantity \(\Delta_{\mathcal L}\) will be the exact gate for whether
latent subdictionary variation reveals residual child orientation.

\begin{lemma}[Observable moment reduction]
\label{lem:moment-reduction}
For the fixed exchangeable code law, the population moments are exactly
\cref{eq:knownlaw-m2,eq:knownlaw-k4}, with
\begin{equation}
\Delta_{\mathcal L}
=\frac{\E\{\lambda_K^2K(n-K)\}}{n(n-1)}.
\label{eq:delta-exact}
\end{equation}
Consequently, \(\Delta_{\mathcal L}>0\) exactly when the law assigns
positive power to a proper nonempty support with positive probability.  For
Bernoulli--Gaussian coding, \cref{eq:bg-moments,eq:p-trace} identify
\begin{equation}
p=\frac{1-\sqrt{1-4u}}2,
\qquad
u=\frac{\operatorname{Tr}_2K_4}{3n},
\qquad
\nu=\frac{\tr M_2-np}{q},
\qquad
\Sigma_2=\frac{M_2-\nu I_q}{p},
\qquad
\Sigma_4=\frac{K_4}{3p(1-p)}.
\tag{3.5}
\label{eq:population-p-nu-inverse}
\end{equation}
\end{lemma}

This lemma explains why the construction uses both second and fourth
moments.  Second-order information locates the block and its quadratic
shape, while fourth-order information separates activation uncertainty and
exposes the cubic orientation invariant below.

The quotient-complete dictionary coordinates are
\begin{equation}
G_2=LL^\top,\qquad
G_3=L^{\otimes3}T_q,\qquad
T_q=\sum_{j=1}^qv_j^{\otimes3}.
\label{eq:quotient-invariants}
\end{equation}
The quadratic invariant \(G_2\) determines shape.  The cubic invariant
\(G_3\) is the first term that retains residual physical orientation modulo
the finite simplex group.  Concretely, expanding the observable fourth-order
projector moment reduces its degree-three cross block to a fixed linear image
of
\[
B_3=qb^{\otimes3}
  +\frac q{q-1}\operatorname{Sym}_3(b\otimes G_2)+G_3,
\]
so the fourth moment carries the cubic orientation coordinate even though
the latent codes are unobserved.  These coordinates are used to analyze the method;
the implemented region is defined directly through the observable moment
equations.

\subsection{A training confidence region, not a point dictionary}

Let \(\widehat M_2\) and \(\widehat M_4\) be fixed Borel robust moment
estimators and set
\begin{equation}
\widehat K_4=\widehat M_4-3\widehat M_2^{\odot2}.
\label{eq:sample-cumulant}
\end{equation}
The coverage theory requires deterministic radii satisfying the uniform
rectangular event
\begin{equation}
\inf_{p,\nu,D}
\Pp_{p,\nu,D}^N\left\{
\|\widehat M_2-M_2\|_F\le\epsilon_{2,N},\quad
\|\widehat K_4-K_4\|_F\le\epsilon_{K,N}
\right\}
\ge1-\alpha_D .
\tag{3.9}
\label{eq:training-rectangle}
\end{equation}
The coordinatewise median-of-means construction in the supplement is one
fully specified, uniformly valid implementation.  It is conservative but
has \(\epsilon_N=(\epsilon_{2,N}^2+\epsilon_{K,N}^2)^{1/2}=O(N^{-1/2})\)
above a fixed block threshold.  Any sharper Borel construction may replace
it only after re-establishing \cref{eq:training-rectangle}; empirical tuning
is not a substitute for this interface.

For the unknown-\(p\) model, invert the rectangle jointly:
\begin{equation}
\widehat{\mathcal K}_{q,s}^{p}
=\left\{
(p,\nu,[D]):
\begin{array}{l}
p\in[p_-,p_+],\ \nu\in[\nu_-,\nu_+],\
[D]\in\mathcal Q^D_{q,s},\\
\|\widehat M_2-\{\nu I_q+p\Sigma_2(D)\}\|_F
\le\epsilon_{2,N},\\
\|\widehat K_4-3p(1-p)\Sigma_4(D)\|_F
\le\epsilon_{K,N}
\end{array}
\right\}.
\tag{3.10}
\label{eq:training-profile}
\end{equation}
The known-law profile replaces the last two population expressions by
\cref{eq:knownlaw-m2,eq:knownlaw-k4} and omits \(p\).  If the robust-moment
block threshold is not met, the method returns the full compact training
shell before evaluating any zero-block radius.

\subsection{One test ball profiles every support and dictionary}

Let
\begin{equation}
\bar Z=T^{-1}\sum_{\ell=1}^TZ_\ell,
\qquad
\tau_{T,q}=\frac{\sigma_+}{\sqrt T}
\sqrt{\chi^2_{q,1-\alpha_T}}.
\label{eq:test-radius}
\end{equation}
For every retained dictionary and fine support, define the box-constrained
profile loss
\begin{equation}
\ell_{\prof}(\bar Z,D,S)
=\min_{x_j\in[\beta_-,\beta_+]}
\left\|\bar Z-\sum_{j\in S}x_jd_j\right\|,
\label{eq:fine-profile-loss}
\end{equation}
and for the separated sheet
\begin{equation}
\ell_{\prof}(\bar Z,D,\partial)
=\min_{\gamma\in\Gamma_A}\|\bar Z-\gamma a\|.
\label{eq:anchor-profile-loss}
\end{equation}
Let \(\zeta=[D,S]\) be the simultaneous dictionary--support orbit.  The
raw joint profile is
\begin{equation}
\mathcal R_{q,r,s}^{p}
=\left\{
(p,\nu,\zeta):
\begin{array}{l}
(p,\nu,[D])\in\widehat{\mathcal K}_{q,s}^{p},\\
S\in\binom{[q]}r\sqcup\{\partial\},\\
\ell_{\prof}(\bar Z,D,S)\le\tau_{T,q}
\end{array}
\right\}.
\tag{3.14}
\label{eq:raw-profile}
\end{equation}
The same Gaussian ball is used for all candidates.  On
\(\|\bar Z-\mu_\star\|\le\tau_{T,q}\), the true dictionary, support, and
coefficient witness survive simultaneously, so no union bound over the
continuum of dictionaries or the finite set of supports is required.

The reported confidence correspondence is the physical image
\begin{equation}
\widehat{\mathfrak C}_{q,r,s}^{p}
=\left\{
\vartheta(D,S):
\exists\,p,\nu\text{ with }(p,\nu,[D,S])\in\mathcal R_{q,r,s}^{p}
\right\}
\tag{3.15}
\label{eq:reported-set}
\end{equation}
when the raw profile is nonempty.  If it is empty, a fixed admissible
null-marked singleton \(\{v^\dagger\}\) makes the random set total and
measurable.  This fallback is reported as an administrative failure status;
it creates neither coverage nor evidence for the separated sheet.

\begin{proposition}[Well-defined exact profile and one-event retention]
\label{prop:profile-well-defined}
The raw profile in \cref{eq:raw-profile} has compact sections, and the
reported object in \cref{eq:reported-set} is a Borel measurable, nonempty,
compact-valued random correspondence.  On the intersection of the training
rectangle \cref{eq:training-rectangle} and the test event
\begin{equation}
\|\bar Z-\mu_\star\|\le\tau_{T,q},
\label{eq:test-truth-event}
\end{equation}
the true dictionary orbit, sheet, support, and coefficient witness belong
to the raw profile, and hence
\(\vartheta_\star\in\widehat{\mathfrak C}_{q,r,s}^{p}\).
\end{proposition}

The proposition is the coverage mechanism in one line.  Profiling enlarges
the output but does not spend confidence level: every candidate is tested
against the same ball, and the true witness needs to survive only once.

\subsection{What the procedure reports}

The object in \cref{eq:reported-set} has a direct decision interpretation.
If both marks survive, the data do not distinguish the coherent and
separated sheets.  If only coherent-sheet targets survive but several
transported supports remain, the group is identified while child identity
is unresolved.  If one transported branch remains, the output has fine
physical resolution up to the learned-dictionary radius.  If only the
separated mark survives, the anchor explanation remains alone.

\begin{enumerate}[label=\textbf{Step \arabic*.},leftmargin=*]
\item Compute \(\widehat M_2,\widehat M_4,\widehat K_4\) and certified radii
      satisfying \cref{eq:training-rectangle}.
\item Construct the complete training-compatible region
      \cref{eq:training-profile}; do not select a representative dictionary.
\item For every retained dictionary orbit, profile every admissible fine
      support and the separated sheet using
      \cref{eq:fine-profile-loss,eq:anchor-profile-loss}.
\item Retain all candidates inside the common test ball and map them through
      \(\vartheta\).
\item Report the surviving marks, physical diameter, and support-branch
      geometry, together with the empty-profile status when applicable.
\end{enumerate}

These are exact compact feasibility operations in fixed dimension.  The
construction is statistically well defined, but we do not claim a
polynomial-time or high-dimensional solver.  A numerical approximation is
honest only if it is a certified outer approximation that cannot remove a
feasible truth.

%% file: sections/05_main_results.tex
\section{Main results}
\label{sec:results}

The results follow the statistical narrative: first the obstruction in the
latent training experiment, then the honest construction, then its
optimality, and finally the circumstances under which the downstream test
task can break the training-side symmetry.

\subsection{Collision-singular training geometry}

Let \(P_j=d_jd_j^\top\), and define
\begin{equation}
T_q=\sum_{j=1}^qv_j^{\otimes3},
\qquad G_2=LL^\top,
\qquad G_3=L^{\otimes3}T_q.
\label{eq:result-invariants}
\end{equation}
For \(\theta=(D,\nu)\), put
\begin{equation}
F(\theta)=(b,P_a,G_2,G_3,\nu)
\label{eq:F-coordinate}
\end{equation}
and let \(\delta_F\) be Euclidean/Frobenius distance in these coordinates.
On the Bernoulli branch, \(\delta_{F,p}^2=|p-p'|^2+\delta_F^2\).
The tensor \(G_2\) records the quadratic block shape; \(G_3\) is the first
invariant that retains residual orientation modulo the finite simplex
permutation group.  Physical rotations are not quotiented out.

\begin{theorem}[Exchangeable sparse-code training geometry]
\label{thm:training-geometry}
Under \cref{ass:frozen-shell}, there exist fixed \(s_0>0\) and
\(0<c<C<\infty\) such that the following statements hold.

\emph{(i) Fixed-law statistical chord.}
For every ordered pair in the same supplied shell,
\begin{equation}
c\delta_F^2(\theta,\theta')
\le H^2(f_\theta,f_{\theta'})
\le\KL(f_\theta\|f_{\theta'})
\le C\delta_F^2(\theta,\theta').
\tag{4.1}
\label{eq:fixed-law-chord}
\end{equation}

\emph{(ii) Unknown-\(p\) statistical chord.}
Uniformly over \(p,p'\in[p_-,p_+]\),
\begin{equation}
c\delta_{F,p}^2
\le H^2(f_{p,\theta},f_{p',\theta'})
\le\KL(f_{p,\theta}\|f_{p',\theta'})
\le C\delta_{F,p}^2.
\tag{4.2}
\label{eq:unknown-p-chord}
\end{equation}
Both all-pair statements remain valid on the full closed fixed shell.

\emph{(iii) Physical quotient inverse.}
After an optimal child permutation \(\pi\),
\begin{align}
\max_j\|z_j-z'_{\pi(j)}\|+\|P_a-P_a'\|_F
\le C\bigl(&\|b-b'\|+\|P_a-P_a'\|_F
\notag\\
&+s^{-1}\|G_2-G_2'\|_F
+s^{-2}\|G_3-G_3'\|_F\bigr).
\tag{4.3}
\label{eq:physical-quotient-inverse}
\end{align}

\emph{(iv) Cubic orientation information.}
At the centered isotropic core in \cref{eq:balanced-hard-core}, let
\(R_t=R\exp(t\Omega)\), where \(\Omega^\top=-\Omega\).  Then
\(f_0\) below denotes the collapsed \(s=0\) mixture with the same regular
nuisances and activation law, and
\begin{equation}
\left.\frac{\partial}{\partial t}f_{s,R_t}\right|_{t=0}
=s^3g_{R,\Omega}
+O_{L^2(f_0^{-1})}(s^4\|\Omega\|_F),
\tag{4.4}
\label{eq:cubic-density-score}
\end{equation}
and \(g_{R,\Omega}\ne0\) for every nonzero quotient-orientation tangent if
and only if \(\Delta_{\mathcal L}>0\).

More generally, at an interior shell point use the point-adapted chart
\begin{equation}
L=sRG,\qquad
L_t=R(sG+tS)e^{t\Omega},\qquad
\dot L=R(S+sG\Omega),
\tag{4.4a}
\label{eq:right-body-chart}
\end{equation}
with \(G=G^\top>0\), \(S=S^\top\), and
\(\Omega^\top=-\Omega\).  After profiling the regular nuisances
\((p,\nu,b,P_a,S)\), with \(p\) omitted for a fixed known law, the
orientation Fisher information satisfies
\begin{equation}
cs^6\|\Omega\|_F^2
\le I_{R\mid\eta,\theta}(\Omega)
\le Cs^6\|\Omega\|_F^2.
\tag{4.5}
\label{eq:profiled-orientation-information}
\end{equation}
The Fisher/QMD statement is restricted to compact interior subbranches.  In
particular, on the centered core,
\begin{equation}
H^2(f_{s,R},f_{s,R'})
\asymp\KL(f_{s,R}\|f_{s,R'})
\asymp s^6d_{O(U)/\mathfrak S_q}^2(R,R').
\label{eq:orientation-chord}
\end{equation}
\end{theorem}

\paragraph{What this result buys.}
The latent dictionary is not a regular root-\(N\) nuisance uniformly through
collision.  Its orientation score is \(s^3\), so a unit orientation change
requires \(Ns^6\) information.  This is also where sparse representation
enters essentially: proper-subdictionary activation makes the cubic score
nonzero, while exact full activation erases it.  The physical inverse in
\cref{eq:physical-quotient-inverse} then converts apparently small invariant
uncertainty into the \(s^{-2}\) loss that governs downstream ray resolution.

\begin{lemma}[Cross-dictionary separation margins]
\label{lem:cross-dictionary-margins}
There are fixed constants \(c_D,g_G^0,C_S,C_G>0\) such that every
dictionary in the supplied shell satisfies
\begin{equation}
\left\|\sum_{j=1}^qh_jd_j\right\|
\ge c_Ds\|h\|_2,
\qquad h\in\R^q.
\label{eq:lower-singular-margin}
\end{equation}
If two candidate dictionaries have physical distance at most \(\rho\), then
different size-\(r\) fine supports are separated by at least
\begin{equation}
m_S(\rho)
=\left[\sqrt2c_D\beta_-s-C_Sr\beta_+\rho\right]_+,
\tag{4.10}
\label{eq:support-margin}
\end{equation}
and a fine-sheet mean and a separated-sheet mean are separated by at least
\begin{equation}
m_G(\rho)
=\left[g_G^0-C_G(r\beta_++\gamma_+)\rho\right]_+.
\tag{4.9}
\label{eq:parent-margin}
\end{equation}
The constants may be chosen so that \(m_S(\rho)>0\) also implies unique
physical child matching across the two dictionaries.
\end{lemma}

\subsection{One honest correspondence, three resolutions}

Write
\(\epsilon_N=(\epsilon_{2,N}^2+\epsilon_{K,N}^2)^{1/2}\).

\begin{theorem}[Conditionally honest, resolution-adaptive support region]
\label{thm:honest-region}
Under \cref{ass:frozen-shell} and the rectangular coverage interface
\cref{eq:training-rectangle}, the correspondence
\(\widehat{\mathfrak C}_{q,r,s}^{p}\) is Borel measurable, nonempty, and
compact-valued.  Uniformly over the joint unknown-\(p\) parameter class,
\begin{equation}
\Pp_{p_\star,\nu_\star,D_\star}^N\!\left[
\Pp_{\mu_\star,\sigma_\star}^T\!\left\{
\vartheta_\star\in\widehat{\mathfrak C}_{q,r,s}^{p}\mid Y^N
\right\}\ge1-\alpha_T
\right]\ge1-\alpha_D.
\tag{4.7}
\label{eq:two-level-coverage}
\end{equation}
Equivalently, for a full parameter \(\xi\), define the measurable test
kernel
\begin{equation}
K_\xi(y)=
\Pp_{\mu(\xi),\sigma(\xi)}^T\!\left\{
z:\vartheta(\xi)\in
\widehat{\mathfrak C}_{q,r,s}^{p}(y,z)
\right\}.
\tag{4.7a}
\label{eq:test-coverage-kernel}
\end{equation}
By independence, with
\(\alpha=1-(1-\alpha_D)(1-\alpha_T)\),
\begin{equation}
\inf_{\xi\in\mathfrak P^p_{q,r}(s)}
\Pp_\xi^{N,T}\!\left\{
\vartheta(\xi)\in\widehat{\mathfrak C}_{q,r,s}^{p}
\right\}\ge1-\alpha.
\tag{4.7c}
\label{eq:marginal-coverage}
\end{equation}

For \(N\) above the fixed median-of-means block threshold, put
\begin{equation}
\rho_N(s)
=C\left[
(s\wedge\epsilon_N)
+\left(s\wedge\frac{\epsilon_N}{s}\right)
+\left(s\wedge\frac{\epsilon_N}{s^2}\right)
\right].
\tag{4.8}
\label{eq:dictionary-radius}
\end{equation}
Every two dictionaries retained by one nonempty training profile have
physical distance at most \(\rho_N(s)\).  If
\(\epsilon_N\le CN^{-1/2}\), then
\begin{equation}
\rho_N(s)
\le C\left(s\wedge\frac1{\sqrt N\,s^2}\right).
\tag{4.8a}
\label{eq:dictionary-radius-rate}
\end{equation}
Below the block threshold, the procedure returns the full compact training
shell and uses \(\rho_N(s)=C_{\rm shell}s\).  The supplied \(s_0\) may be
chosen so this branch obeys the same order as
\cref{eq:dictionary-radius-rate}.

For every realization,
\begin{align}
\diam_{\pr}(\widehat{\mathfrak C}_{q,r,s}^{p})
\le{}&C\rho_N(s)
\notag\\
&+Cs\,\1\{2\tau_{T,q}\ge m_S(\rho_N(s))\}
\notag\\
&+C\,\1\{2\tau_{T,q}\ge m_G(\rho_N(s))\}.
\tag{4.11}
\label{eq:three-scale-upper-bound}
\end{align}
Thus the exact same correspondence has order-one, order-\(s\), and
order-\(\rho_N(s)\) physical resolutions, governed respectively by
\begin{equation}
I_G^{(r)}=\frac{Tr^2\beta_-^2}{\sigma_+^2},
\qquad I_S=\frac{T\beta_-^2s^2}{\sigma_+^2},
\qquad I_D=Ns^6.
\label{eq:theorem-three-gates}
\end{equation}
These are separate gates; no \(NT\) product gate is asserted.
\end{theorem}

\paragraph{What this result buys.}
The method is not merely a larger interval around a plug-in estimate.  It
changes the granularity of the reported sparse statement while retaining the
physical truth: inconclusive across sheets, parent-resolved but
child-ambiguous, or fine-support resolved.  The two-level guarantee is
stronger than a single unconditional average: for most training samples, the
conditional test procedure covers.  It uses no observed training codes and
no support-labelled calibration pairs.

\subsection{The three gates are unavoidable}

\begin{lemma}[Two-point honesty implies diameter]
\label{lem:honesty-diameter}
Let two parameters \(\xi_0,\xi_1\) have physical target distance \(d\).
If a random compact set \(\widehat C\) has coverage at least \(1-\alpha\)
at both endpoints, then
\begin{equation}
\E_{\xi_0}\diam_{\pr}(\widehat C)
\ge d\left\{
1-2\alpha-\operatorname{TV}
(\Pp_{\xi_0}^{N,T},\Pp_{\xi_1}^{N,T})
\right\}_+.
\label{eq:honesty-diameter-lemma}
\end{equation}
\end{lemma}

\begin{theorem}[Fixed-shell minimax lower bound and resolved optimality]
\label{thm:fixed-shell-minimax}
Assume
\((1-\alpha_D)(1-\alpha_T)>1/2\) and set
\(\alpha=1-(1-\alpha_D)(1-\alpha_T)<1/2\).  Let
\begin{equation}
\mathfrak H_\alpha\{\mathfrak P^p_{q,r}(s)\}
=\left\{\widehat C:
\inf_{\xi\in\mathfrak P^p_{q,r}(s)}
\Pp_\xi^{N,T}\{\vartheta(\xi)\in\widehat C\}
\ge1-\alpha\right\},
\label{eq:honest-class}
\end{equation}
where \(\widehat C\) ranges over Borel nonempty compact-valued
correspondences, and define
\begin{equation}
\mathcal R_{N,T}^{(q,r)}(s)
=\inf_{\widehat C\in\mathfrak H_\alpha}
\sup_{\xi\in\mathfrak P^p_{q,r}(s)}
\E_\xi\diam_{\pr}(\widehat C).
\label{eq:minimax-diameter-risk}
\end{equation}
There exist fixed positive constants and transition thresholds such that
\begin{equation}
\mathcal R_{N,T}^{(q,r)}(s)
\ge c\max\left[
\1\{I_G^{(r)}\le c_G\},
\ s\1\{I_S\le c_S\},
\ s\wedge\frac1{\sqrt N\,s^2}
\right].
\tag{4.13}
\label{eq:three-gate-lower-bound}
\end{equation}
When the parent and support gates exceed their fixed upper thresholds, the
correspondence in \cref{thm:honest-region}, equipped with the conservative
certified radii in the supplement or any replacement satisfying
\(\epsilon_N\le CN^{-1/2}\), gives
\begin{equation}
\mathcal R_{N,T}^{(q,r)}(s)
\asymp s\wedge\frac1{\sqrt N\,s^2}.
\tag{4.14}
\label{eq:resolved-minimax-rate}
\end{equation}
This is constant-factor optimality on one supplied fixed shell in the
resolved regime.  It is not a sharp-transition-constant theorem, a matching
upper theorem throughout every transition band, a multiway
candidate-list-size theorem, or adaptation over unknown \(s\).
\end{theorem}

\begin{corollary}[The known-dictionary/learned-dictionary resolution gap]
\label{cor:known-learned-gap}
Consider a sequence on the supplied shell for which \(I_G^{(r)}\to\infty\)
and \(I_S\to\infty\), while \(I_D=Ns^6\) remains bounded.  With the
dictionary known, the common-ball test profile eventually separates the
fine supports.  Over the learned-dictionary class, however, every uniformly
honest physical-support correspondence has worst-case expected diameter at
least \(cs\).  Thus atom-level resolution in a selected dictionary can be
strictly finer than the physical resolution justified after latent
dictionary learning.
\end{corollary}

\paragraph{What this result buys.}
The method's coarser outputs are not conservatism that a cleverer procedure
can uniformly remove.  Three least-favourable comparisons reproduce the
three operating states, and the learned-dictionary rate matches the upper
bound once the test-side gates are open.  In particular,
\cref{cor:known-learned-gap} isolates the inferential failure that motivates
the paper without claiming that every concrete plug-in estimator must
undercover in every parameter cell.

\subsection{When test data can break the training symmetry}

At a centered-core fine support \(S\), define the coefficient-profiled
orientation secant
\begin{equation}
\chi_{\rm tan}(D,S,x;\Omega)
=\frac1s\inf_{\dot x\in\R^r}
\|\dot D_S[\Omega]x+D_S\dot x\|
=\frac1s\left\|
P_{\operatorname{span}(D_S)}^\perp\dot D_S[\Omega]x
\right\|.
\tag{4.15}
\label{eq:tangent-secant}
\end{equation}

\begin{theorem}[Task geometry can preserve or break the training singularity]
\label{thm:task-symmetry}
For every fixed \(2\le r\le q-1\), the exact local test efficient
information is
\begin{equation}
I_{\Omega\mid x}^{\rm test}
=\frac{Ts^2}{\sigma^2}
\chi_{\rm tan}^2(D,S,x;\Omega).
\tag{4.16}
\label{eq:test-efficient-information}
\end{equation}

For a finite statement, fix interior \(p_0,\nu_0\), the anchor, a support
\(S_0\), every nonorientation dictionary coordinate, and known test noise
\(\sigma\).  Let
\begin{equation}
D_\phi
=D\{b=0,L=\lambda_\star sR\exp(\phi\Omega),P_a=P_{a_0}\},
\qquad |\phi|\le h_0,
\label{eq:restricted-orbit}
\end{equation}
where the chart is injective modulo the finite simplex group.  Let
\(\mathcal X_{\rm orb}\Subset(\beta_-,\beta_+)^r\) and define
\begin{equation}
\mathfrak P_{\rm orb}
=\{(p_0,\nu_0,D_\phi,S_0,x,\sigma):
|\phi|\le h_0,\ x\in\mathcal X_{\rm orb}\}.
\label{eq:restricted-class}
\end{equation}
Assume the target and coefficient-profiled test means have two-sided
secants
\begin{align}
c_Vs|\phi-\phi'|
&\le d_H^{\pr}\{\vartheta(D_\phi,S_0),
\vartheta(D_{\phi'},S_0)\}
\le C_Vs|\phi-\phi'|,
\tag{4.17}
\label{eq:target-secant}\\
cs\chi_s|\phi-\phi'|
&\le\inf_{x'\in\mathcal X_{\rm orb}}
\|\mu(D_\phi,S_0,x)-\mu(D_{\phi'},S_0,x')\|
\le Cs\chi_s|\phi-\phi'|.
\tag{4.18}
\label{eq:test-mean-secant}
\end{align}
Let \(\mathcal R_{N,T}^{\rm task}\) be the minimax expected physical
diameter over Borel compact correspondences with marginal coverage
\(1-\alpha\) uniformly on this restricted orbit.  Then
\begin{equation}
\mathcal R_{N,T}^{\rm task}
(s,\mathcal X_{\rm orb},\sigma)
\asymp
s\wedge\frac{s}{
\sqrt{Ns^6+T\chi_s^2s^2/\sigma^2}}.
\tag{4.19}
\label{eq:restricted-task-rate}
\end{equation}
This finite matching statement is restricted to
\(\mathfrak P_{\rm orb}\).  Equation
\cref{eq:test-efficient-information}, not
\cref{eq:restricted-task-rate}, is the general-\(r\), full-model
conclusion.
\end{theorem}

\begin{corollary}[Two-atom amplitude contrast]
\label{cor:two-atom-contrast}
Let \(r=2\) and write
\(x_1=\bar\beta+d/2\), \(x_2=\bar\beta-d/2\).  Fix \(d_0\ne0\) and a
sign-preserving compact coefficient shell on which
\(|d-d_0|\le\epsilon_d\le|d_0|/2\).  On the orbit rotating \(v_1-v_2\)
while fixing \(v_1+v_2\),
\begin{equation}
\inf_{\bar\beta',d'\in\R}
\|\mu_{\phi,\bar\beta,d}
-\mu_{\phi',\bar\beta',d'}\|
=\frac{\lambda_\star s|d|\|v_1-v_2\|}{2}
|\sin(\phi-\phi')|.
\tag{4.20}
\label{eq:two-atom-profile}
\end{equation}
Hence
\(\chi_s\asymp\lambda_\star\|v_1-v_2\||d_0|/2\), and the training/test
crossover occurs at
\begin{equation}
|d_0|\asymp\sigma\sqrt{N/T}\,s^2.
\tag{4.21}
\label{eq:amplitude-crossover}
\end{equation}
The separate exact slice \(d=0\) is test invariant and remains limited by
the \(Ns^6\) training information.
\end{corollary}

\paragraph{What this result buys.}
The theorem tells an experimenter when collecting more test replicates can
help.  A zero task secant means that coefficient profiling absorbs the
orientation perturbation, so only dictionary training resolves the physical
rays.  A nonzero secant contributes additive test information, and
\cref{cor:two-atom-contrast} gives an explicit crossover.  The finite
matching rate is intentionally not generalized beyond the declared orbit.

%% file: sections/06_proof_ideas.tex
\section{Why the rates arise}
\label{sec:proof-ideas}

This section gives the proof logic in the same order as the statistical
claims.  Complete density expansions, uniform envelopes, quotient arguments,
measurability details, and least-favourable constructions belong in the
technical supplement; the purpose here is to expose the mechanism rather
than hide it behind the moment algebra.

\subsection{Why orientation appears at cubic order}

Conditional on a training support indicator \(I\), the observation \(Y\) is
Gaussian with covariance
\begin{equation}
C_I=\nu I_q+\lambda_K\sum_{j=1}^n I_jP_j .
\label{eq:conditional-training-covariance}
\end{equation}
A conditional Wick expansion followed by exchangeable averaging gives
\cref{lem:moment-reduction}.  The important proof choice is to aggregate
over supports before taking absolute values.  The aggregate density through
degree three depends on the child power sums
\begin{equation}
B_k=\sum_{j=1}^q z_j^{\otimes k},\qquad k=1,2,3.
\label{eq:child-power-sums}
\end{equation}
At \(b=0\), \(B_1=0\); \(B_2\) depends on \(LL^\top\) and is blind to
residual right orientation.  The first orientation-sensitive term is
\begin{equation}
B_3=L^{\otimes3}T_q.
\label{eq:cubic-power-sum}
\end{equation}
Its coefficient is \(\Delta_{\mathcal L}\).  This proves both cancellation
through degree two and the exact proper-subdictionary criterion for a
nonzero cubic score.  A common Gaussian cross-weight envelope controls all
finite-mixture components, and a fourth-order composed remainder yields the
directed Hellinger/KL upper chord.  A uniformly injective collapsed moment
operator supplies the reverse Hellinger bound.

The inverse step is geometric.  Polar whitening recovers the quadratic
shape from \(G_2\).  The stabilizer identity
\begin{equation}
\operatorname{Stab}_{O(U)}(T_q)=\mathfrak S_q
\label{eq:tensor-stabilizer}
\end{equation}
then recovers residual orientation from \(G_3\) modulo the child
permutations and no larger group.  Because \(G_2\) and \(G_3\) have sizes
\(s^2\) and \(s^3\), their physical inverse carries the \(s^{-1}\) and
\(s^{-2}\) factors in \cref{eq:physical-quotient-inverse}.  Finally, a
point-adapted right-body chart separates the regular shape directions from
orientation.  Score--moment duality prevents the cubic direction from being
absorbed by the regular nuisance tangent space, proving the profiled
\(s^6\) information in
\cref{eq:profiled-orientation-information}.

\subsection{Why one profile is honest}

The coverage argument is intentionally simpler than the geometry.  Let
\(\mathcal E_D\) be the training rectangle in
\cref{eq:training-rectangle}, and let
\(\mathcal E_T=\{\|\bar Z-\mu_\star\|\le\tau_{T,q}\}\).
On \(\mathcal E_D\), the true \((p_\star,\nu_\star,[D_\star])\) remains in
the training profile.  On \(\mathcal E_T\), the true support and coefficient
vector witness a feasible test profile for that same dictionary.  Hence
\begin{equation}
\mathcal E_D\cap\mathcal E_T
\subseteq
\{\vartheta_\star\in\widehat{\mathfrak C}_{q,r,s}^{p}\}.
\label{eq:truth-retention-inclusion}
\end{equation}
The two-level coverage statement follows because
\(\Pp(\mathcal E_D)\ge1-\alpha_D\) and, conditional on every training
sample, \(\Pp(\mathcal E_T)\ge1-\alpha_T\).  There is no union bound over
supports: the event only needs the true witness to survive.

For measurability, the parameter shell and coefficient boxes are compact,
the moment maps and profile losses are continuous, and the graph of the raw
profile is Borel with compact sections.  A measurable-projection theorem
then gives a random compact physical image.  The fallback only totalizes the
map on an empty raw profile; the event in
\cref{eq:truth-retention-inclusion} never invokes it.

On one nonempty training profile, two retained moment images differ by at
most twice the rectangle radii.  The all-pair chord and quotient inverse
give \(\rho_N(s)\).  The cross-dictionary margins in
\cref{lem:cross-dictionary-margins} then have a transparent consequence:
two candidates inside the same radius-\(\tau_{T,q}\) test ball cannot
coexist if their means are more than \(2\tau_{T,q}\) apart.  Excluding the
separated sheet removes the order-one term; excluding wrong fine supports
removes the order-\(s\) term; the remaining variation is the transported
dictionary radius.

\subsection{Why the three lower bounds match the three decisions}

\Cref{lem:honesty-diameter} converts indistinguishability into unavoidable
confidence diameter.  We apply it to three pairs in the same parameter
class.

\begin{enumerate}[leftmargin=*,label=\textbf{Pair \arabic*.}]
\item A fine-versus-anchor pair has identical training distribution,
order-one physical separation, and test KL of order \(I_G^{(r)}\).
This produces the parent gate.

\item Two size-\(r\) supports differing by one child have identical training
distribution, physical separation of order \(s\), and test KL of order
\(I_S\).  This produces the child-support gate.

\item A compensated orientation pair moves an active physical ray by
order \(s|h|\) while modifying the coefficients so that the entire test mean
is unchanged.  Its joint KL is therefore training-only and bounded by
\(CNs^6h^2\).  Taking
\[
|h|\asymp\min\{1,(\sqrt N\,s^3)^{-1}\}
\]
produces
\(s\wedge(\sqrt N\,s^2)^{-1}\).
\end{enumerate}

The third pair explains why increasing \(T\) cannot universally repair a
closed dictionary gate: along that least-favourable path, the test law is
exactly the same.  Once the first two gates are open, the upper radius of
\cref{thm:honest-region} and this third lower bound coincide up to constants.

\subsection{Why some test signals nevertheless help}

An orientation perturbation changes the test mean by
\(\dot D_S[\Omega]x\).  Coefficient perturbations can absorb precisely the
component in \(\operatorname{span}(D_S)\).  Projection onto the orthogonal
complement therefore yields the efficient secant
\(\chi_{\rm tan}\) and the local information identity
\cref{eq:test-efficient-information}.  On the restricted orbit, training and
test KL add by independence:
\begin{equation}
\KL_{\rm joint}
\asymp
\left(Ns^6+\frac{T\chi_s^2s^2}{\sigma^2}\right)
|\phi-\phi'|^2.
\label{eq:joint-orbit-kl}
\end{equation}
The target secant converts the admissible orientation width into physical
diameter, producing \cref{eq:restricted-task-rate}.  For two active atoms,
equal amplitudes cancel the transverse component exactly; a nonzero contrast
reveals it.  This is why the explicit crossover depends on the amplitude
difference rather than on reconstruction signal-to-noise ratio alone.

%% file: sections/08_discussion.tex
\section{Discussion}
\label{sec:discussion}

The paper addresses a decision problem created by learning the representation
before using it.  Near coherent-atom collision, dictionary uncertainty does
not simply add a regular variance term to a known-design confidence interval.
It can change the finest physical support statement that is identifiable.
The proposed correspondence responds by retaining all training-compatible
dictionaries until after the test support is profiled.  Its main gain is
therefore calibrated granularity: it refuses unstable leaf precision when
the dictionary gate is closed, yet avoids the permanent conservatism of an
always-group rule when the joint information opens all three gates.

The phase diagram also changes how additional data should be collected.  A
closed parent or support gate is a test-side limitation.  A closed dictionary
gate is a calibration limitation and, along compensated directions, cannot
be repaired by arbitrarily many repeated measurements of the same test
signal.  The task secant refines that conclusion: asymmetric coefficients can
create transverse test information, whereas symmetric signals remain
training-limited.  Thus the theory gives a diagnostic for allocating effort
between dictionary training, test replication, and signal diversity.

\paragraph{Scope.}
The current theorem is fixed-dimensional.  It treats one supplied coherent
affine-simplex block, one separated anchor, one externally supplied positive
shell \(s\), positive beta-min coefficients, and fixed
\(2\le r\le q-1\).  It does not learn a hierarchy, adapt over an unknown
collision scale, treat arbitrary coherent frames, allow growing \(q\), or
provide a polynomial-time high-dimensional solver.  The finite matching task
theorem applies only to the explicitly declared one-dimensional orbit; its
general-\(r\) content is the local efficient-information identity.

\paragraph{Computation and finite-sample sharpness.}
Exact compact profiling defines a valid statistical benchmark and the
supplementary median-of-means rectangle provides a uniformly covering
fixed-\(q\) construction with the required \(N^{-1/2}\) order above its block
threshold.  That construction is conservative.  A practically useful
implementation must replace it with a sharper certified radius and must
compute an outer approximation that cannot delete a feasible truth.
Uncertified nonconvex optimization, bootstrap inflation around one selected
dictionary, or empirical quantile tuning would not inherit the honesty
theorem.

\paragraph{Relation to applications.}
Sparse array localization and blind spectral unmixing motivate physical
rays, coherent groups, finite calibration, and resolution-adaptive
statements.  The present results do not claim a formal embedding of either
application.  The synthetic affine-simplex experiment is nevertheless a
meaningful theory demonstration: it isolates the exact train--test
phenomenon, provides physical support ground truth, and makes every
information gate independently controllable.  A subsequent application
paper should let the sensing geometry determine the hierarchy and should
prove or validate the bridge rather than attach a generic reconstruction
benchmark.

\paragraph{Outlook.}
Three extensions are especially natural.  First, certified sharp profiling
would turn the optimal statistical construction into an operational method.
Second, multiple coherent blocks and data-driven but confidence-valid
hierarchies would extend adaptive resolution beyond one supplied parent.
Third, an amortized sparse solver could learn to approximate the exact
profile while a deterministic outer certificate preserves honesty.  These
extensions should be judged against the same principle as the present paper:
the output must state no more physical resolution than the full data can
support.

%% file: appendices/A_main_proofs.tex
\subsection{Proof of the collision-singular training geometry}
\begin{proof}[Proof of \cref{thm:training-geometry}]

\par\medskip\noindent\textbf{Exact exchangeable moment reduction}\label{exact-exchangeable-moment-reduction}

Conditional on the support indicator \(I\), the training observation is Gaussian with covariance

\[
C_I=\nu I_q+\lambda_K\sum_{j=1}^nI_jP_j.
\]

Exchangeability gives, for \(i\ne j\),

\[
E(\lambda_KI_j)=a_1,
\quad
E(\lambda_K^2I_j)=a_{2d},
\quad
E(\lambda_K^2I_iI_j)=a_{2o}.
\tag{5.1}
\]

Taking expectations proves (3.1). Conditional Wick expansion gives

\[
K_4=3\{E(C_I^{\odot2})-(EC_I)^{\odot2}\}.
\]

Separating diagonal and off-diagonal index pairs,

\[
E\left(\lambda_K\sum_jI_jP_j\right)^{\odot2}
=a_{2o}\Sigma_2^{\odot2}
+(a_{2d}-a_{2o})\Sigma_4,
\]

which proves (3.2). Finally,

\[
a_{2d}-a_{2o}
=\frac{E\{\lambda_K^2K(n-K)\}}{n(n-1)}.
\tag{5.2}
\]

Thus \(\Delta_{\mathcal L}=0\) exactly when the code law puts no positive power on a proper nonempty support. On every fixed law satisfying (2.6), the triangular map

\[
(M_2,K_4)\longleftrightarrow(\nu,\Sigma_2,\Sigma_4)
\tag{5.3}
\]

is uniformly bi-Lipschitz. The Bernoulli specialization follows by setting \(a_1=a_{2d}=p\), \(a_{2o}=p^2\), and then using (3.4)--(3.5).

\par\medskip\noindent\textbf{Exchangeable subset aggregation}\label{exchangeable-subset-aggregation}

The density proof must aggregate subsets before taking absolute values. Let \(H_1,\ldots,H_q\) belong to a commutative symmetric tensor algebra, \(Q_m=\sum_jH_j^{\odot m}\), and let \(C\) be uniform over all \(h\)-subsets of \([q]\). Put \(X_C=\sum_{j\in C}H_j\) and

\[
\pi_m(h)=\frac{(h)_m}{(q)_m}.
\]

Partitioning ordered index tuples by equality pattern gives

\[
E_hX_C=\pi_1Q_1,
\tag{5.4}
\]

\[
E_hX_C^{\odot2}
=(\pi_1-\pi_2)Q_2+\pi_2Q_1^{\odot2},
\tag{5.5}
\]

and

\[
\begin{aligned}
E_hX_C^{\odot3}={}&
(\pi_1-3\pi_2+2\pi_3)Q_3\\
&+3(\pi_2-\pi_3)Q_1\odot Q_2
+\pi_3Q_1^{\odot3}.
\end{aligned}
\tag{5.6}
\]

These identities are exact for every \(h\). They are the reason the componentwise order-\(s\) displacement does not determine the mixture-level orientation information.

Condition on total support size \(K=k\) and anchor indicator \(A\in\{0,1\}\). The active-child count is \(h=k-A\), with weight

\[
\rho_{k,A}=\frac{\binom q{k-A}}{\binom{q+1}k}.
\tag{5.7}
\]

Define

\[
\mathcal A_k=\{A\in\{0,1\}:0\le k-A\le q\}.
\]

Every sum over \((k,A)\) below means \(A\in\mathcal A_k\); no density is defined for an invalid stratum. Equivalently, use the conventions \(\binom qm=0\) outside \(m\in\{0,\ldots,q\}\), and omit every zero-weight stratum. In (5.10), set \(e_m=0\) for \(m\notin\{0,\ldots,q\}\) and \(e_0=1\).

For \(H_j=P_j-P_u\), the base and perturbation covariances are

\[
C_{k,A,0}=\nu I_q+\lambda_k(hP_u+AP_a),
\qquad
X_C=\lambda_k\sum_{j\in C}H_j.
\]

Taylor expansion of the already aggregated stratum through total child degree three depends only on

\[
Q_1,Q_2,Q_3,Q_1^{\odot2},Q_1\odot Q_2,Q_1^{\odot3}.
\tag{5.8}
\]

After the projector chart is substituted, these are triangular analytic functions of the three child power sums

\[
B_1=\sum_jz_j,
\qquad B_2=\sum_jz_j^{\otimes2},
\qquad B_3=\sum_jz_j^{\otimes3}.
\tag{5.9}
\]

The exact characteristic function, useful for checking that no support stratum has been omitted, is

\[
\widehat f(t)
=e^{-\nu\|t\|^2/2}
\sum_{k=0}^n\frac{\pi_k}{\binom nk}
\left[
e_k(w_{1,k},\ldots,w_{q,k})
+w_{a,k}e_{k-1}(w_{1,k},\ldots,w_{q,k})
\right],
\tag{5.10}
\]

where

\[
w_{j,k}=\exp\{-\lambda_k(d_j^\top t)^2/2\},
\qquad
w_{a,k}=\exp\{-\lambda_k(d_a^\top t)^2/2\}.
\]

\par\medskip\noindent\textbf{Cubic observed orientation score}\label{cubic-observed-orientation-score}

At the centered isotropic core,

\[
d_j(s,R)=
\sqrt{1-\lambda_\star^2s^2}\,u
+\lambda_\star sRv_j.
\]

Write \(\alpha=u^\top t\), \(x=P_Ut\), and \(a_j=x^\top Rv_j\). For a uniform \(h\)-subset \(C\), let

\[
A_C=\sum_{j\in C}a_j,
\qquad B_C=\sum_{j\in C}a_j^2.
\]

Simplex identities and (5.4)--(5.6) give

\[
E_hA_C=0,
\]

\[
E_h(A_CB_C)
=\frac{h(q-h)}{q(q-1)}
(R^{\otimes3}T_q)[x,x,x],
\tag{5.11}
\]

and

\[
E_hA_C^3
=\frac{h(q-h)(q-2h)}{q(q-1)(q-2)}
(R^{\otimes3}T_q)[x,x,x].
\tag{5.12}
\]

The orientation-dependent part of the normalized stratum characteristic function is therefore cubic in \(s\). Along \(R_\epsilon=R\exp(\epsilon\Omega)\), its leading Fourier coefficient is

\[
\widehat g_{R,\Omega}(t)
=\frac{\lambda_\star^3\Delta_{\mathcal L}}2
\alpha\,
R^{\otimes3}(\Omega\cdot T_q)[x,x,x]
+O(\|t\|^6).
\tag{5.13}
\]

The anchor-averaged coefficient used here follows from the exact identity

\[
\sum_{A=0}^1\rho_{k,A}
\frac{(k-A)\{q-(k-A)\}}{q(q-1)}
=\frac{k(n-k)}{n(n-1)}.
\tag{5.14}
\]

Finally,

\[
\|\Omega\cdot T_q\|_F^2
=3\left(\frac q{q-1}\right)^3\|\Omega\|_F^2.
\tag{5.15}
\]

Hence \(g_{R,\Omega}\ne0\) for every nonzero quotient-orientation tangent whenever \(\Delta_{\mathcal L}>0\).  Conversely, if \(\Delta_{\mathcal L}=0\), the nonnegative identity in (2.5) forces every positive-variance activation stratum to have \(K\in\{0,n\}\).  The empty stratum is independent of the child orientation, and at the centered isotropic core the full-stratum covariance depends on the children only through the tight-frame sum \(\sum_jP_j\), which is also orientation invariant.  The complete mixture density is then orientation invariant and \(g_{R,\Omega}=0\).  This proves the equivalence. Weighted Gaussian derivative bounds in the next subsection turn (5.13) into the norm expansion (4.4). A matching lower information bound can also be obtained by differentiating (3.2), noting that the orientation derivative of \(\Sigma_2\) vanishes at the isotropic core while the derivative of \(\Sigma_4\) contains \(s^3R^{\otimes3}(\Omega\cdot T_q)\), and applying the score--moment Cauchy--Schwarz inequality under uniform eighth moments.

\par\medskip\noindent\textbf{Gaussian envelope and the full density chord}\label{gaussian-envelope-and-the-full-density-chord}

\par\smallskip\noindent\textit{Lemma: --- common cross-weight envelope}\label{lemma-common-cross-weight-envelope}

Suppose the eigenvalues of \(A\) lie in a fixed compact subset of \((0,\infty)\), and

\[
\|A^{-1/2}(B-A)A^{-1/2}\|_{\mathrm{op}}\le\rho<1/2.
\]

For every derivative order \(0\le m\le5\),

\[
\int
\frac{|D^m\gamma_A[H_1,\ldots,H_m]|^2}{\gamma_B}
\le C_m\prod_{j=1}^m\|H_j\|_F^2.
\tag{5.16}
\]

\textbf{Proof.} A covariance derivative is a degree-\(2m\) polynomial in \(y\), with uniformly bounded coefficients, times \(\gamma_A(y)\). The first derivative is explicitly

\[
D\gamma_A[H](y)
=\frac{\gamma_A(y)}2
\left\{y^\top A^{-1}HA^{-1}y
-\operatorname{tr}(A^{-1}H)\right\}.
\tag{5.16a}
\]

Repeated differentiation proves the polynomial-times-Gaussian form used for orders two through five. The Gaussian quotient has precision

\[
2A^{-1}-B^{-1}
\succeq
\left(2-\frac1{1-\rho}\right)A^{-1}\succ0.
\]

All required polynomial moments are therefore uniformly finite. The same calculation controls integral Taylor remainders and one derivative of those remainders. \(\square\)

For each \((k,A)\) stratum, define the already aggregated full density map

\[
\mathcal G_{k,A}(z_1,\ldots,z_q;P_a,\nu)
=\binom qh^{-1}\sum_{|C|=h}
\gamma_{\nu I+\lambda_k\{AP_a+\sum_{j\in C}P(d(z_j))\}}.
\]

Let \(\mathcal P_{\le3,k,A}\) be its multivariate Taylor polynomial at the collapsed child point, truncated by total child degree three. Set

\[
\mathcal P_{\le3}
=\sum_{k,A}\pi_k\rho_{k,A}\mathcal P_{\le3,k,A},
\qquad
R_{\ge4}=f-\mathcal P_{\le3}.
\]

Straight paths between optimally aligned shell endpoints may leave the condition-number cone, but remain in an enlarged \(O(s)\) labelled cap. Taylor's formula and Lemma 5.1 give there

\[
\left\|
\frac{D_zR_{\ge4}[\dot z]}{\sqrt{f_{\theta'}}}
\right\|_2
\le Cs^3\max_j\|\dot z_j\|,
\tag{5.17}
\]

\[
\left\|
\frac{D_{(P_a,\nu)}R_{\ge4}[\dot P_a,\dot\nu]}
{\sqrt{f_{\theta'}}}
\right\|_2
\le Cs^4(\|\dot P_a\|_F+|\dot\nu|).
\tag{5.18}
\]

The physical inverse in Subsection 5.5 implies

\[
s^3\max_j\|z_j-z'_{\pi(j)}\|
\le C\{s^3\|\Delta b\|
+s^2\|\Delta G_2\|
+s\|\Delta G_3\|\}
\le Cs\delta_F.
\tag{5.19}
\]

Integrating (5.17)--(5.18) proves

\[
\left\|
\frac{R_{\ge4}(\theta)-R_{\ge4}(\theta')}
{\sqrt{f_{\theta'}}}
\right\|_2
\le Cs\delta_F(\theta,\theta').
\tag{5.20}
\]

Choose \(s_0\) and a fixed local-noise radius \(r_\nu>0\) so that, whenever \(|\nu-\nu'|\le r_\nu\), every path covariance and either endpoint covariance obey the relative bound in Lemma 5.1. This is possible because the training noise is at least \(\nu_-\), all code powers are bounded, the anchor patch and enlarged child cap are \(O(s)\), and there are finitely many strata. No minimum mixture weight is introduced: for component remainders \(r_\ell\), endpoint components \(q_\ell\), and weights \(w_\ell\),

\[
\frac{(\sum_\ell w_\ell r_\ell)^2}
{\sum_\ell w_\ell q_\ell}
\le
\sum_\ell w_\ell\frac{r_\ell^2}{q_\ell}.
\tag{5.21}
\]

For the low-order part, (5.4)--(5.9) reduce every coefficient to

\[
B_1,B_2,B_3,B_1^{\otimes2},
\operatorname{Sym}(B_1\otimes B_2),B_1^{\otimes3}.
\]

The exact identities

\[
B_1=qb,
\tag{5.22}
\]

\[
B_2=qb^{\otimes2}+\frac q{q-1}G_2,
\tag{5.23}
\]

\[
B_3=qb^{\otimes3}
+\frac q{q-1}\operatorname{Sym}_3(b\otimes G_2)
+G_3
\tag{5.24}
\]

and Lemma 5.1 yield

\[
\left\|
\frac{\mathcal P_{\le3}(\theta)-
\mathcal P_{\le3}(\theta')}{\sqrt{f_{\theta'}}}
\right\|_2
\le C\delta_F(\theta,\theta').
\tag{5.25}
\]

Equations (5.20) and (5.25), and the same two estimates with the endpoints exchanged, give the two directed local-noise bounds

\[
\max\left\{
\int\frac{(f_\theta-f_{\theta'})^2}{f_{\theta'}},
\int\frac{(f_\theta-f_{\theta'})^2}{f_\theta}
\right\}
\le C\delta_F^2(\theta,\theta')
\quad\text{if }|\nu-\nu'|\le r_\nu.
\tag{5.25a}
\]

If \(|\nu-\nu'|>r_\nu\), align the child labels and apply log-sum to the common fixed-law component weights. Compactness of all component covariance pairs gives \(\operatorname{KL}(f_\theta\|f_{\theta'})\le K\), whereas \(\delta_F^2\ge r_\nu^2\). Hence

\[
\operatorname{KL}(f_\theta\|f_{\theta'})
\le Kr_\nu^{-2}\delta_F^2(\theta,\theta').
\tag{5.25b}
\]

Supplement S9 writes the component pairing, zero-weight convention, and both ordered directions continuously.

For the Hellinger lower chord, take

\[
\Psi(Y)=\{\operatorname{vech}(YY^\top),
\operatorname{symvec}(Y^{\otimes4})\}.
\]

Uniform Gaussian-mixture eighth moments and Cauchy--Schwarz imply

\[
\|E_\theta\Psi-E_{\theta'}\Psi\|
\le C H(f_\theta,f_{\theta'}).
\tag{5.26}
\]

Indeed,

\[
E_\theta\Psi-E_{\theta'}\Psi
=\int\Psi(\sqrt f_\theta-\sqrt f_{\theta'})
(\sqrt f_\theta+\sqrt f_{\theta'}),
\]

and the second Cauchy--Schwarz factor is uniformly bounded by the eighth moment envelope. This is the required moment-to-Hellinger lemma.

The explicit fixed-law inverse and the uniform collapsed moment secant in Supplement S8 give the continuous lower chain

\[
\delta_F
\le C\|\Delta(\nu,\Sigma_2,\Sigma_4)\|
\le C\|\Delta(M_2,K_4)\|
\le C\|\Delta E\Psi\|
\le C H(f_\theta,f_{\theta'}).
\tag{5.26a}
\]

Together with \(H^2\le\operatorname{KL}\), (5.25a)--(5.26a) prove (4.1).

For unknown Bernoulli \(p\), the lower chord follows from (3.3)--(3.5). For the upper chord, hold \(p\) fixed and use (4.1), then change \(p\) at a fixed dictionary. With support-indicator weights

\[
w_I(p)=p^{|I|}(1-p)^{n-|I|},
\]

weighted Cauchy--Schwarz gives

\[
\int
\frac{\{f_{p,\nu,D}-f_{p',\nu,D}\}^2}
{f_{p',\nu,D}}
\le
\sum_I\frac{\{w_I(p)-w_I(p')\}^2}{w_I(p')}
\le C|p-p'|^2.
\tag{5.27}
\]

All weights are uniformly positive on the fixed branch. To make the endpoint argument explicit, set \(f_0=f_{p,\nu,D}\), \(f_1=f_{p,\nu',D'}\), and \(f_2=f_{p',\nu',D'}\). Since only the Bernoulli weights change between \(f_1\) and \(f_2\), compactness of \([p_-,p_+]\) gives \(cf_1\le f_2\le Cf_1\) pointwise. Therefore, in the local-noise region,

\[
\begin{aligned}
\chi^2(f_0\|f_2)
&\le2\int\frac{(f_0-f_1)^2}{f_2}
+2\int\frac{(f_1-f_2)^2}{f_2}\\
&\le C\chi^2(f_{p,\nu,D}\|f_{p,\nu',D'})
+C\sum_I\frac{\{w_I(p)-w_I(p')\}^2}{w_I(p')}\\
&\le C\{\delta_F^2+|p-p'|^2\}.
\end{aligned}
\tag{5.27a}
\]

The same calculation is repeated with endpoints exchanged. In the far-noise region, generalized log-sum bounds the joint latent-component KL by the Bernoulli weight KL plus a compact weighted Gaussian KL; the fixed noise separation absorbs the latter constant into \(\delta_{F,p}^2\). This proves (4.2) without asserting pointwise comparability between different dictionaries.

\par\medskip\noindent\textbf{Quotient inverse and nuisance-projected orientation information}\label{quotient-inverse-and-nuisance-projected-orientation-information}

Equations (5.22)--(5.24) give a triangular inverse from child power sums to \((b,G_2,G_3)\). The separated anchor is identified jointly with these coordinates by the second/fourth-moment pair. In a local anchor chart \(a(w)=(a_0+w)/\|a_0+w\|\), with \(w\in a_0^\perp\), let \(X=tu+x\), \(x\in U\), and \(H(\dot w)=a_0\dot w^\top+\dot w a_0^\top\). For increments \(\dot B=(\dot B_1,\dot B_2,\dot B_3)\), define

\[
\begin{aligned}
\mathsf S_{\dot B}(t,x)
={}&2t\dot B_1[x]+\dot B_2[x,x]
-t^2\operatorname{tr}\dot B_2-t\dot B_3[x,I_U],\\
\mathsf Q_{\dot B}(t,x)
={}&4t^3\dot B_1[x]+6t^2\dot B_2[x,x]
-2t^4\operatorname{tr}\dot B_2\\
&+4t\dot B_3[x,x,x]-6t^3\dot B_3[x,I_U].
\end{aligned}
\tag{5.28}
\]

With \(h_0(t,x)=\tfrac12(t+e_1^\top x)^2\), the collapsed operator is

\[
\mathscr A_q(\dot B,\dot w,\dot\nu)
=\left(
\dot\nu,
\mathsf S_{\dot B}+H(\dot w)[X,X],
\mathsf Q_{\dot B}+2h_0H(\dot w)[X,X]
\right).
\tag{5.28a}
\]

Supplement S8 proves coefficient by coefficient that \(\mathscr A_q\) is injective on \(U\oplus\operatorname{Sym}^2(U)\oplus\operatorname{Sym}^3(U) \oplus a_0^\perp\oplus\mathbb R\). Hence, for fixed \(q\),

\[
\|\mathscr A_q h\|\ge\sigma_q\|h\|,
\qquad \sigma_q>0.
\tag{5.28b}
\]

After the algebraic quotient alignment from the physical inverse below, the exact all-pair secant satisfies

\[
\mathcal R(\theta)-\mathcal R(\theta')
=\mathscr A_q\{\Xi(\theta)-\Xi(\theta')\}
+\operatorname{Rem}_q(\theta,\theta'),
\qquad
\|\operatorname{Rem}_q\|
\le C_qs\|\Xi(\theta)-\Xi(\theta')\|,
\tag{5.28c}
\]

where \(\Xi=(B_1,B_2,B_3,w,\nu)\), \(\mathcal R=(\nu,\Sigma_2,\Sigma_4)\), and \(\|\Delta\Xi\|\asymp\delta_F\). Choosing \(s_0\le\sigma_q/(2C_q)\) yields the uniform same-shell moment secant used in (5.26a). This step uses the quotient inverse only for algebraic alignment, so the statistical moment lower bound is not used circularly.

For the physical inverse, take the left polar decomposition

\[
L=P_LO_L,
\qquad P_L=(LL^\top)^{1/2}.
\]

Functional calculus on the normalized compact shape class gives

\[
\|P_L-P_L'\|_F
\le Cs^{-1}\|G_2-G_2'\|_F.
\tag{5.29}
\]

Normalize the cubic tensor:

\[
W(L)=P_L^{-\otimes3}G_3=O_L^{\otimes3}T_q.
\]

The cubic \(x\mapsto T_q[x,x,x]\) has its positive spherical maxima exactly at the simplex vertices. Hence every orthogonal stabilizer of \(T_q\) permutes those vertices, and conversely every simplex permutation stabilizes the tensor. Therefore

\[
\operatorname{Stab}_{O(U)}(T_q)=\mathfrak S_q.
\tag{5.29a}
\]

Identity (5.15) gives a uniformly injective orbit differential. Local quotient inverse bounds follow from the constant-rank/local-embedding theorem; compactness separates pairs outside a finite union of diagonal quotient charts. This proves the global estimate

\[
\min_{\pi\in\mathfrak S_q}
\|O_L-O_L'\Pi_\pi\|_F
\le C\left[
s^{-3}\|G_3-G_3'\|_F
+s^{-2}\|G_2-G_2'\|_F
\right].
\tag{5.30}
\]

Combining (5.29)--(5.30) proves (4.3).

For nuisance-projected orientation information use exactly the right-body chart (4.4a). Direct differentiation gives

\[
DG_2[h]=sR(SG+GS)R^\top,
\tag{5.31}
\]

and

\[
DG_3[h]
=\mathcal S_{R,G,s}[S]
+(sRG)^{\otimes3}(\Omega\cdot T_q),
\qquad
\|\mathcal S_{R,G,s}[S]\|_F\le Cs^2\|S\|_F.
\tag{5.31a}
\]

The Lyapunov map \(S\mapsto SG+GS\) is uniformly coercive on the normalized compact shape class. After its contribution is recovered from \(DG_2\), invertibility of \(G^{\otimes3}\) and (5.15) recover the orientation contribution from \(DG_3\). Adding the separate coordinates therefore yields

\[
\|D(p,F)(\theta)[h]\|^2
\asymp
|\dot p|^2+|\dot\nu|^2+\|\dot b\|^2+\|\dot P_a\|_F^2
+s^2\|S\|_F^2+s^6\|\Omega\|_F^2,
\tag{5.31b}
\]

with \(\dot p\) omitted for a fixed law. This is the required \(s\)-uniform nonisotropic statement; global injectivity alone would not supply the weighted constant.

Let \(\ell_{\theta,h}=Df_\theta[h]/f_\theta\). Uniform eighth moments and the score identity

\[
D E_\theta\Psi[h]
=E_\theta\{(\Psi-E_\theta\Psi)\ell_{\theta,h}\}
\]

imply \(\|D(M_2,K_4)[h]\|^2\le C I_\theta(h)\). Differentiating the all-pair moment chord and using (5.31b) gives the reverse lower bound, while the differentiated density expansion gives the upper bound. Consequently

\[
c\|h\|_{\mathrm{mix},\theta}^2
\le I_\theta(h)
\le C\|h\|_{\mathrm{mix},\theta}^2.
\tag{5.31c}
\]

Taking the infimum over all nuisance tangents proves the lower half of (4.5), and setting them to zero proves its upper half. Supplement S10 supplies the inner/enlarged-tube tangent scope and all derivative details.
\end{proof}

\subsection{Proof of measurability, conditional honesty, and adaptive diameter}
\begin{proof}[Proof of \cref{prop:profile-well-defined,thm:honest-region}]

\par\medskip\noindent\textbf{Measurability, truth retention, and conditional coverage}\label{measurability-truth-retention-and-conditional-coverage}

Let

\[
\psi(Y)=\{\operatorname{vech}(YY^\top),
\operatorname{symvec}(Y^{\otimes4})\}\in\mathbb R^{d_q}.
\]

For each of the second- and fourth-order orthonormal symmetric-coordinate systems, split the sample deterministically into \(B_k\asymp\log(2d_k/\alpha_k)\) blocks, take every coordinate block mean, and take its lower median under a fixed tie convention. Decode the result as \((\widehat M_2,\widehat M_4)\), then apply the explicit polynomial cumulant map

\[
(A,B)\mapsto(A,B-3A^{\odot2}).
\]

Uniform eighth moments, the scalar median-of-means inequality, a finite coordinate union bound, fixed-dimensional norm equivalence, and the explicit raw-moment-to-cumulant propagation in Supplement S13 prove (3.9). For sample sizes below the block threshold, the procedure returns the full compact training class.

Let \(\widetilde{\mathcal D}_{q,s}\) be the closed, canonically signed labelled shell and form the compact associated quotient

\[
\mathcal M^p_{q,r,s}
=\frac{[p_-,p_+]\times[\nu_-,\nu_+]
\times\widetilde{\mathcal D}_{q,s}
\times\left(\binom{[q]}r\sqcup\{\partial\}\right)}
{\mathfrak S_q}.
\tag{5.31d}
\]

The training projection, the compact-fiber profile value, and the physical target map are continuous on this quotient. The graph of the raw set is a Borel intersection of two closed-in-parameter sublevel constraints and has compact sections. For every open hit set, its graph intersection has sigma-compact sections; the Arsenin--Kunugui projection theorem \citep[Theorem~18.18]{Kechris1995} therefore makes the hit event Borel. In particular, raw-profile nonemptiness is Borel. Taking the continuous physical image on that event and the fixed singleton \(\{v^\dagger\}\) on its complement proves that (3.15) is a Borel map into \(\mathcal K(\mathfrak V_q)\) with its Hausdorff Borel sigma field. No global atom ordering or quotient selector is used. Supplement S11 gives the standalone random-correspondence lemma.

Let

\[
E_D=\{(p_\star,\nu_\star,[D_\star])
\in\widehat{\mathcal K}_{q,s}^{p}\}.
\]

Equation (3.9) gives \(P(E_D)\ge1-\alpha_D\). Conditional on any training realization in \(E_D\), the event

\[
\|\bar Z-\mu_\star\|\le\tau_{T,q}
\]

retains the true support and coefficient witness and has probability at least \(1-\alpha_T\), uniformly over \(\sigma\le\sigma_+\). This proves (4.7). No union bound over candidate supports is used: the truth survives on one common test-noise event.

For two members of the same nonempty training profile, (3.3)--(3.5) and the branch inequality

\[
|p(1-p)-p'(1-p')|
=|p-p'|\,|1-p-p'|
\ge(1-2p_+)|p-p'|
\]

give

\[
|p-p'|+
\delta_F\{(D,\nu),(D',\nu')\}
\le C\epsilon_N.
\tag{5.32}
\]

The physical inverse (4.3) proves (4.8).

\par\medskip\noindent\textbf{Cross-dictionary margins and the upper phase diagram}\label{cross-dictionary-margins-and-the-upper-phase-diagram}

Supplement S12 first proves a uniform child-frame singular-value bound. There is a displayed \(c_D>0\), depending only on the frozen shell, such that

\[
\left\|\sum_{j=1}^qh_jd_j\right\|
\ge c_Ds\|h\|_2,
\qquad h\in\mathbb R^q.
\tag{5.32a}
\]

For two different size-\(r\) supports, the coefficient-difference vector has at least two exclusive entries of magnitude at least \(\beta_-\). Thus the same-dictionary margin is at least \(\sqrt2c_D\beta_-s\). Transporting \(r\) canonically signed atoms across dictionary distance \(\rho\) costs at most \(C_Sr\beta_+\rho\), giving the explicit version of (4.10):

\[
m_S(\rho)
=\left[\sqrt2c_D\beta_-s-C_Sr\beta_+\rho\right]_+.
\tag{5.32b}
\]

For the parent margin, put \(c_u=\sqrt{1-C_0^2s_0^2}\) and \(h_a=(u-e_1)/\sqrt2\). Since \(h_a\perp a_0\), canonical signs give \(|h_a^\top a|\le C_as_0\), whereas every fine mean satisfies

\[
h_a^\top\mu(D,S,x)
\ge\frac r{\sqrt2}
\{\beta_-c_u-\beta_+C_0s_0\}.
\]

Shrink \(s_0\) once so that

\[
g_G^0
:=\frac r{\sqrt2}\{\beta_-c_u-\beta_+C_0s_0\}
-\gamma_+C_as_0
\ge\frac{r\beta_-}{2\sqrt2}.
\tag{5.32c}
\]

Cross-dictionary transport then yields

\[
m_G(\rho)
=\left[g_G^0-C_G(r\beta_++\gamma_+)\rho\right]_+,
\tag{5.32d}
\]

which is (4.9) after renaming fixed constants.

Any two retained candidates have fitted means within \(2\tau_{T,q}\). Consequently, a fine and separated candidate cannot coexist if \(2\tau_{T,q}<m_G(\rho)\).  If \(m_S(\rho)=0\), the support indicator in (4.11) is automatically active and every pair of fine ray sets lies in the common \(O(s)\) cap, so its distance is \(O(s)\) without any matching.  If \(m_S(\rho)>0\), the transport constant \(C_S\) is chosen large enough that \(\rho<c_{\mathrm{match}}s\); quotient separation then gives a unique child matching.  Distinct transported supports cannot coexist when \(2\tau_{T,q}<m_S(\rho)\), while rays on the same transported branch differ by \(O(\rho)\).  Two separated candidates also differ by \(O(\rho)\) through their anchor projectors.  These exhaustive deterministic cases prove (4.11).

The operational labels ``correct parent'' and ``correct leaf branch'' are used only under a fine truth and on the coverage event. An empty raw profile has the singleton fallback and may have zero diameter, but it certifies neither event.
\end{proof}

\subsection{Proof of the minimax lower bound and resolved optimality}
\begin{proof}[Proof of \cref{lem:honesty-diameter,thm:fixed-shell-minimax,cor:known-learned-gap}]

\par\medskip\noindent\textbf{Matching fixed-shell lower pairs}\label{matching-fixed-shell-lower-pairs}

For two joint parameters \(\xi_0,\xi_1\) with target distance \(d\), any confidence correspondence having coverage \(1-\alpha\) at both endpoints satisfies

\[
E_{\xi_0}\operatorname{diam}_{\mathrm{pr}}(\widehat C)
\ge
d\{1-2\alpha-
\operatorname{TV}(P_{\xi_0}^{N,T},P_{\xi_1}^{N,T})\}_+.
\tag{5.33}
\]

To verify (5.33), let \(A_i\) be the event that the confidence set contains the target at endpoint \(i\). Under \(P_0\),

\[
P_0(A_0\cap A_1)
\ge1-2\alpha-\operatorname{TV}(P_0,P_1).
\]

On this intersection the reported diameter is at least \(d\). Multiplying and taking expectation proves the claim.

Three fixed-\(p_0\) submodels of the unknown-\(p\) class yield (4.13).

\begin{enumerate}
\def\labelenumi{\arabic{enumi}.}
\item
  \textbf{Parent pair.} On the hard core choose \(x=\beta_-\mathbf1_S\), \(\sigma=\sigma_+\), and \(\gamma_\star=a_0^\top\mu_F\in\operatorname{int}(\Gamma_A)\), the last inclusion following from (2.10a). The training laws coincide and

  \[
  \operatorname{KL}_{N,T}
  =\frac{T}{2\sigma_+^2}
  \|(I-P_{a_0})\mu_F\|^2
  \le C\frac{Tr^2\beta_-^2}{\sigma_+^2}
  =CI_G^{(r)}.
  \tag{5.33a}
  \]

  Meanwhile the physical target distance is at least \(d_G=\sin\{\pi/4-\arcsin(\lambda_\star s_0)\}>0\) after shrinking \(s_0\).
\item
  \textbf{Wrong-support pair.} Keep the hard-core dictionary fixed, choose two supports sharing \(r-1\) atoms, and use the common coefficient \(\beta_-\). If \(i,j\) are the exchanged atoms, then

  \[
  \|\mu_S-\mu_{S'}\|^2
  =\beta_-^2\|d_i-d_j\|^2
  =2\beta_-^2\lambda_\star^2s^2\frac q{q-1},
  \tag{5.33b}
  \]

  so the joint KL is
  \(T\beta_-^2\lambda_\star^2s^2q/[(q-1)\sigma_+^2]
  \le CI_S\). The physical target distance equals

  \[
  \sqrt{1-\left\{1-\lambda_\star^2s^2q/(q-1)\right\}^2}
  \asymp s.
  \tag{5.33c}
  \]
\item
  \textbf{Compensated orientation pair.} At the centered isotropic hard core \(L=\lambda_\star sR\), with equal active coefficients, set \(w_S=\sum_{j\in S}v_j\). For \(q\ge4\), \(r\ge2\), there exists a nonzero skew generator \(\Omega\) that fixes \(w_S\) but moves at least one active vertex. Indeed, \(\|w_S\|^2=r(q-r)/(q-1)>0\), while the centered active differences span a nonzero subspace of \(w_S^\perp\); choose a skew rotation on a two-plane in \(w_S^\perp\) containing one such difference. Normalize it so that \(\max_{j\in S}\|\Omega v_j\|=1\).

  Here is the required target-secant detail. At \(h=0\), the finitely many active rays are pairwise separated by at least \(c_0s\). Choose a fixed \(h_0\) below both the finite-quotient injectivity radius and the radius on which every cross-label match remains at least \(c_0s/2\). On this chart the same-label match is optimal. The normalized generator moves at least one active ray with projective speed bounded below by \(cs\), while all active rays have speed at most \(Cs\). Taylor expansion, uniformly over the finite set of active labels, therefore gives

  \[
  cs|h|\le
  d_H^{\mathrm{pr}}
  \{\vartheta(D_0,S),\vartheta(D_h,S)\}
  \le Cs|h|,
  \qquad |h|\le h_0.
  \]

  The path \(R\exp(h\Omega)\) leaves the complete test law identical and has training KL at most \(CNs^6h^2\) by the fixed-law chord (4.1), applied to the law induced by \(p=p_0\). Taking \[
  |h|\asymp\min\{h_0,(\sqrt N\,s^3)^{-1}\}
  \] gives the lower term \(s\wedge(\sqrt N\,s^2)^{-1}\).
\end{enumerate}

Pinsker's inequality and (5.33) complete the lower bound.  For the matching
upper bound put
\[
 R_D(s,N)=s\wedge(\sqrt N\,s^2)^{-1}.
\]
By (4.8a), \(\rho_N\le C_\rho R_D\).  First choose the supplied shell
radius \(s_0\) so that
\(m_G(C_\rho s)\ge g_G^0/2\) for every \(s\le s_0\), and then choose a
fixed \(c_0>0\) so small that
\(C_\rho c_0<c_{\rm match}\) and
\[
\sqrt2c_D\beta_-s-C_Sr\beta_+C_\rho c_0s
\ge \frac{c_D\beta_-}{\sqrt2}s.
\]
The fixed upper threshold for the parent gate makes
\(2\tau_{T,q}<g_G^0/2\), so the order-one indicator in (4.11) always
vanishes in the resolved regime.  If \(R_D\ge c_0s\), the possible support
term satisfies \(Cs\le (C/c_0)R_D\) and is absorbed without requiring
support matching.  If \(R_D<c_0s\), then
\(\rho_N<C_\rho c_0s\), so the displayed inequality gives a positive
order-\(s\) support margin and unique matching.  The fixed upper threshold
for \(I_S\) makes \(2\tau_{T,q}<m_S(\rho_N)\), hence the support indicator
vanishes.  In both cases (4.11) is \(O(R_D)\).  The honest correspondence is therefore
an admissible member of the minimax class with worst-case expected diameter
at most \(CR_D\).  Combined with the compensated-orientation lower pair,
this proves (4.14) for the same class, target, and physical loss.

For the known-dictionary comparison in
\cref{cor:known-learned-gap}, set the training radius to zero and run the
same common-ball profile at the true dictionary.  The fixed-dictionary
support margin is \(\sqrt2c_D\beta_-s\), while
\(\tau_{T,q}/s=O(I_S^{-1/2})\); hence \(I_S\to\infty\) leaves only the true
fine-support branch.  Likewise \(I_G^{(r)}\to\infty\) excludes the separated
sheet.  Its physical diameter therefore tends to zero.  Over the
learned-dictionary class, the compensated-orientation pair instead retains
an order-\(s\) indistinguishable physical displacement whenever \(Ns^6\)
stays bounded.  This proves the asserted gap.
\end{proof}

\subsection{Proof of task-dependent symmetry breaking}
\begin{proof}[Proof of \cref{thm:task-symmetry,cor:two-atom-contrast}]

\par\medskip\noindent\textbf{Task-aligned orientation proof}\label{task-aligned-orientation-proof}

Let

\[
w_{S,x}=\sum_{j\in S}x_jv_j,
\qquad
\mathcal V_S^0
=\left\{\sum_{j\in S}c_jv_j:\sum_{j\in S}c_j=0\right\}.
\]

At the centered core, an orientation generator changes the test mean by

\[
\dot\mu_R=\lambda_\star sR\Omega w_{S,x}.
\]

Coefficient perturbations with zero sum generate \(\lambda_\star sR\mathcal V_S^0\); a nonzero sum creates an order-one axial component and cannot cancel a pure order-\(s\) transverse rotation. Orthogonal projection gives (4.15)--(4.16).

A zero local margin need not imply a finite invariant chord: an \(O(sh^2)\) residual may remain. This is why \cref{thm:task-symmetry} explicitly assumes both the physical-target secant (4.17) and the test-mean secant (4.18).

For the upper bound, restrict the training moment profile (3.10) and the raw
test profile (3.14) to \(\mathfrak P_{\rm orb}\), and use the known-noise
test radius
\[
\tau_T=\frac{\sigma}{\sqrt T}\sqrt{\chi^2_{q,1-\alpha_T}}.
\]
Report the physical images of all retained orbit points, with the same fixed
fallback on an empty profile.  The training rectangle and the one common
Gaussian test-ball event retain the truth, so the restricted correspondence
has the required uniform marginal coverage.  For any two retained
parameters, put \(h=\phi-\phi'\).  The fixed-law invariant chord on the
injective orbit and the
two-sided target secant give

\[
s^3|h|\lesssim N^{-1/2}
\]

from their common training profile, while the triangle inequality for the
two fitted means and the lower test secant give

\[
s\chi_s|h|\lesssim\sigma T^{-1/2}
\]

from the test profile.  Hence
\[
|h|\lesssim
\min\left\{h_0,(Ns^6)^{-1/2},
\left(T\chi_s^2s^2/\sigma^2\right)^{-1/2}\right\}.
\]
The minimum of the last two terms is within a fixed factor of the inverse
square root of their sum. Since the target secant is comparable to
\(s|h|\), the worst-case diameter of this honest restricted profile has the
upper rate in (4.19). For the lower bound, choose

\[
|h|\asymp
\min\left\{h_0,
(Ns^6+T\chi_s^2s^2/\sigma^2)^{-1/2}\right\},
\]

and use the coefficient path supplied by the upper secant. Training and test KL add by independence, while target separation is \(cs|h|\). Equation (5.33) proves the matching lower bound.

For the explicit two-atom contrast, write

\[
e=v_1+v_2,
\qquad f=v_1-v_2,
\]

and choose a rotation fixing \(e\) while sending \(f\) to \(\cos(h)f+\sin(h)g\), with \(g\perp\operatorname{span}\{e,f\}\) and \(\|g\|=\|f\|\). Exact coefficient profiling gives

\[
\inf_{\bar\beta',d'\in\mathbb R}
\|\mu_{\phi,\bar\beta,d}
-\mu_{\phi',\bar\beta',d'}\|
=\frac{\lambda_\star s|d|\|f\|}{2}
|\sin(\phi-\phi')|.
\tag{5.34}
\]

Choose \(h_0\) below the injectivity radius of the finite simplex quotient. The selected generator fixes \(v_1+v_2\) but moves at least one active ray; continuity and compactness then give the two-sided physical target secant (4.17). On \(\mathcal X_{\mathrm{orb}}(d_0,\epsilon_d)\), the constrained infimum is bounded below by (5.34), while the same \((\bar\beta,d)\) candidate bounds it above by \(C\lambda_\star s|d|\|f\||\phi-\phi'|\). Since \(|d|\asymp|d_0|\) on the sign-preserving shell, this verifies (4.18). For the separate \(d=0\) slice, the rotation fixes the entire test mean. Together these facts prove (4.20)--(4.21).
\end{proof}

\begin{center}\rule{0.5\linewidth}{0.5pt}\end{center}

%% file: appendices/B_technical_details.tex
\hypertarget{projector-chart-through-total-degree-three}{%
\section{Technical detail S1: projector chart through total degree three}\label{projector-chart-through-total-degree-three}}

For \(z\in U\), write

\[
d(z)=c(z)u+z,
\qquad c(z)=\sqrt{1-\|z\|^2},
\]

and

\[
H(z)=P\{d(z)\}-P_u.
\]

Since \(c(z)^2=1-\|z\|^2\) exactly and

\[
c(z)=1-\frac12\|z\|^2+O(\|z\|^4),
\]

the total-degree expansion is

\[
H(z)=A_1(z)+A_2(z)+A_3(z)+R_4(z),
\tag{S1.1}
\]

where

\[
A_1(z)=uz^\top+zu^\top,
\tag{S1.2}
\]

\[
A_2(z)=zz^\top-\|z\|^2P_u,
\tag{S1.3}
\]

\[
A_3(z)=-\frac12\|z\|^2(uz^\top+zu^\top).
\tag{S1.4}
\]

In this particular chart the first omitted monomial is actually degree five, but the weaker uniform bounds

\[
\|R_4(z)\|_F\le C\|z\|^4,
\qquad
\|DR_4(z)\|_{\mathrm{op}}\le C\|z\|^3
\tag{S1.5}
\]

are sufficient and remain stable under the multivariate density composition. All constants are uniform on the enlarged child cap.

\begin{center}\rule{0.5\linewidth}{0.5pt}\end{center}

\hypertarget{one-stratum-and-its-total-degree-density-polynomial}{%
\section{Technical detail S2: one stratum and its total-degree density polynomial}\label{one-stratum-and-its-total-degree-density-polynomial}}

Fix a support-size/anchor stratum \((k,A)\), set \(h=k-A\), and write

\[
C_0=\nu I_q+\lambda_k(hP_u+AP_a).
\]

Only \((k,A)\) with \(0\le h\le q\) and positive stratum weight are included. We use \(\binom qm=0\) for invalid \(m\), and in elementary symmetric-polynomial formulas use \(e_m=0\) outside \(m\in\{0,\ldots,q\}\), \(e_0=1\).

For a child subset \(C\), define

\[
X_m(C)=\lambda_k\sum_{j\in C}A_m(z_j),
\qquad m=1,2,3,
\]

and

\[
X_R(C)=\lambda_k\sum_{j\in C}R_4(z_j).
\]

Let \(\mathcal D_m=D^m\gamma_{C_0}\) denote the \(m\)-th covariance derivative at the base covariance. Composing the Gaussian covariance Taylor series with (S1.1), and retaining monomials by total child degree, gives

\[
\mathcal P_{0,C}=\gamma_{C_0},
\tag{S2.1}
\]

\[
\mathcal P_{1,C}=\mathcal D_1[X_1(C)],
\tag{S2.2}
\]

\[
\mathcal P_{2,C}
=\mathcal D_1[X_2(C)]
+\frac12\mathcal D_2[X_1(C),X_1(C)],
\tag{S2.3}
\]

and

\[
\begin{aligned}
\mathcal P_{3,C}={}&
\mathcal D_1[X_3(C)]
+\mathcal D_2[X_1(C),X_2(C)]\\
&+\frac16\mathcal D_3[X_1(C),X_1(C),X_1(C)].
\end{aligned}
\tag{S2.4}
\]

Every term omitted from

\[
\mathcal P_{\le3,C}=\sum_{m=0}^3\mathcal P_{m,C}
\]

has total child degree at least four. This statement includes terms produced by the nonlinear projector chart; it is not merely a third-order covariance Taylor approximation.

Let \(E_h\) denote the uniform average over all \(h\)-subsets. The aggregated stratum polynomial is

\[
\mathcal P_{\le3,k,A}=E_h\mathcal P_{\le3,C}.
\]

\begin{center}\rule{0.5\linewidth}{0.5pt}\end{center}

\hypertarget{complete-subset-aggregation-table}{%
\section{Technical detail S3: complete subset-aggregation table}\label{complete-subset-aggregation-table}}

Put

\[
S_m^{(a)}=\sum_{j=1}^qA_m(z_j)^{\odot a},
\qquad
S_m=\sum_{j=1}^qA_m(z_j).
\]

For two different homogeneous degrees, also put

\[
S_{12}=\sum_{j=1}^qA_1(z_j)\odot A_2(z_j).
\]

Let

\[
\pi_a=\frac{(h)_a}{(q)_a}.
\]

Uniform subset averaging gives the following exact table. Powers of \(\lambda_k\) multiply each row according to the number of \(X\) factors and are suppressed in the middle column.

\begin{longtable}[]{@{}
  >{\raggedleft\arraybackslash}p{(\columnwidth - 6\tabcolsep) * \real{0.3077}}
  >{\raggedright\arraybackslash}p{(\columnwidth - 6\tabcolsep) * \real{0.2308}}
  >{\raggedright\arraybackslash}p{(\columnwidth - 6\tabcolsep) * \real{0.2308}}
  >{\raggedright\arraybackslash}p{(\columnwidth - 6\tabcolsep) * \real{0.2308}}@{}}
\toprule\noalign{}
\begin{minipage}[b]{\linewidth}\raggedleft
Total degree
\end{minipage} & \begin{minipage}[b]{\linewidth}\raggedright
Density term
\end{minipage} & \begin{minipage}[b]{\linewidth}\raggedright
Exact subset aggregate
\end{minipage} & \begin{minipage}[b]{\linewidth}\raggedright
Child invariant dependency
\end{minipage} \\
\midrule\noalign{}
\endhead
\bottomrule\noalign{}
\endlastfoot
0 & \(\gamma_{C_0}\) & \(\gamma_{C_0}\) & \((P_a,\nu)\) \\
1 & \(\mathcal D_1[X_1]\) & \(\pi_1\mathcal D_1[S_1]\) & \(B_1\) \\
2 & \(\mathcal D_1[X_2]\) & \(\pi_1\mathcal D_1[S_2]\) & \(B_2\) \\
2 & \(\frac12\mathcal D_2[X_1,X_1]\) & \(\frac12\mathcal D_2[(\pi_1-\pi_2)S_1^{(2)}+\pi_2S_1^{\odot2}]\) & \(B_2,B_1^{\otimes2}\) \\
3 & \(\mathcal D_1[X_3]\) & \(\pi_1\mathcal D_1[S_3]\) & \(B_3\) \\
3 & \(\mathcal D_2[X_1,X_2]\) & \(\mathcal D_2[(\pi_1-\pi_2)S_{12}+\pi_2S_1\odot S_2]\) & \(B_3,\operatorname{Sym}(B_1\otimes B_2)\) \\
3 & \(\frac16\mathcal D_3[X_1,X_1,X_1]\) & \(\frac16\mathcal D_3[(\pi_1-3\pi_2+2\pi_3)S_1^{(3)}+3(\pi_2-\pi_3)S_1\odot S_1^{(2)}+\pi_3S_1^{\odot3}]\) & \(B_3,\operatorname{Sym}(B_1\otimes B_2),B_1^{\otimes3}\) \\
\end{longtable}

The mixed second-derivative row follows by separating equal and unequal indices:

\[
E_h\left\{
\sum_{i\in C}A_1(z_i)\odot
\sum_{j\in C}A_2(z_j)
\right\}
=(\pi_1-\pi_2)S_{12}+\pi_2S_1\odot S_2.
\tag{S3.1}
\]

The last row is the equality-pattern identity for ordered triples. Thus the table contains every degree-at-most-three density term and no componentwise triangle inequality has been used.

\hypertarget{reduction-of-the-homogeneous-projector-sums}{%
\subsection{Reduction of the homogeneous projector sums}\label{reduction-of-the-homogeneous-projector-sums}}

Let \(\mathscr L(z)=uz^\top+zu^\top\). Then

\[
S_1=\mathscr L(B_1).
\tag{S3.2}
\]

Moreover,

\[
S_2=\sum_jz_jz_j^\top-\operatorname{tr}(B_2)P_u,
\tag{S3.3}
\]

so \(S_2\) is a fixed linear image of \(B_2\). Since \(A_1\) is linear,

\[
S_1^{(2)}=\sum_j\mathscr L(z_j)^{\odot2}
\tag{S3.4}
\]

is also a fixed linear image of \(B_2\).

Let

\[
c_3=B_3[\cdot,I_U]=\sum_j\|z_j\|^2z_j.
\]

Then

\[
S_3=-\frac12\mathscr L(c_3),
\tag{S3.5}
\]

and both

\[
S_{12}=\sum_jA_1(z_j)\odot A_2(z_j),
\qquad
S_1^{(3)}=\sum_jA_1(z_j)^{\odot3}
\tag{S3.6}
\]

are fixed linear images of \(B_3\). Products of the aggregate sums give the displayed symmetrized products of \(B_1,B_2,B_3\). Therefore the table reduces exactly to

\[
B_1, B_2, B_3, B_1^{\otimes2},
\operatorname{Sym}(B_1\otimes B_2),B_1^{\otimes3}.
\tag{S3.7}
\]

No degree-one or degree-two residual-orientation coordinate exists in this list: \(B_1=qb\), while \(B_2\) depends on \(L\) only through \(LL^\top\). Residual orientation first enters through the \(B_3\) term \(L^{\otimes3}T_q\).

\hypertarget{conversion-to-quotient-coordinates}{%
\subsection{Conversion to quotient coordinates}\label{conversion-to-quotient-coordinates}}

The exact affine-simplex identities are

\[
B_1=qb,
\tag{S3.8}
\]

\[
B_2=qb^{\otimes2}+\frac q{q-1}G_2,
\tag{S3.9}
\]

\[
B_3
=qb^{\otimes3}
+\frac q{q-1}\operatorname{Sym}_3(b\otimes G_2)
+G_3.
\tag{S3.10}
\]

On the shell,

\[
\|\Delta B_1\|\le C\|\Delta b\|,
\tag{S3.11}
\]

\[
\|\Delta B_2\|
\le C\{s\|\Delta b\|+\|\Delta G_2\|\},
\tag{S3.12}
\]

\[
\|\Delta B_3\|
\le C\{s^2\|\Delta b\|
+s\|\Delta G_2\|+\|\Delta G_3\|\}.
\tag{S3.13}
\]

Because \(B_m=O(s^m)\), every product row in the aggregation table obeys the same \(O(\delta_F)\) endpoint chord. Smooth changes in \((P_a,\nu)\) are handled by differentiating the base covariance derivative operators and using the common Gaussian envelope. This proves the low-order endpoint bound in manuscript equation (5.25), term by term.

\begin{center}\rule{0.5\linewidth}{0.5pt}\end{center}

\hypertarget{full-composed-remainder}{%
\section{Technical detail S4: full composed remainder}\label{full-composed-remainder}}

Let

\[
\mathcal G_{k,A}(z_1,\ldots,z_q;P_a,\nu)
=E_h\gamma_{C_0+X_1+X_2+X_3+X_R}
\]

be the already aggregated full stratum density, and let

\[
R_{\ge4,k,A}
=\mathcal G_{k,A}-\mathcal P_{\le3,k,A}.
\]

Before aggregating within this stratum, let \(R_{\ge4,C}=\mathcal G_C-\mathcal P_{\le3,C}\), and let \(q_{C,\theta'}\) be the corresponding endpoint subset density. For every collection of subsetwise remainder increments \(r_C\), weighted Cauchy--Schwarz gives the pointwise inequality

\[
\frac{\{E_h r_C\}^2}{E_hq_{C,\theta'}}
\le E_h\frac{r_C^2}{q_{C,\theta'}}.
\tag{S4.0}
\]

Thus the common Gaussian envelope may first be applied subset by subset and then averaged; no inverse subset probability or minimum component weight is introduced inside a stratum.

Bundle all labelled child coordinates into \(z=(z_1,\ldots,z_q)\). Apply the integral Taylor formula to the derivative map \(D_z\mathcal G_{k,A}\). The exact Fréchet identity is

\[
D_zR_{\ge4,k,A}(z)[\dot z]
=\frac12\int_0^1(1-t)^2
D_z^4\mathcal G_{k,A}(tz)
[z,z,z,\dot z],dt.
\tag{S4.1}
\]

For example, (S4.1) gives \(4z^3\dot z\) when \(\mathcal G(z)=z^4\), fixing the coefficient that would be lost by differentiating a schematic multiindex remainder. Every integrand is a finite combination of covariance derivatives of order at most four, projector derivatives, and three powers of child coordinates. Lemma 5.1 of the manuscript therefore gives

\[
\left\|
\frac{D_zR_{\ge4,k,A}[\dot z]}
{\sqrt{q_{k,A,\theta'}}}
\right\|_2
\le Cs^3\max_j\|\dot z_j\|,
\tag{S4.2}
\]

where \(q_{k,A,\theta'}\) is the corresponding endpoint stratum density.

For \(\eta=(P_a,\nu)\), apply the fourth-order remainder formula directly to \(D_\eta\mathcal G\):

\[
D_\eta R_{\ge4,k,A}(z)[\dot\eta]
=\frac1{6}\int_0^1(1-t)^3
D_z^4D_\eta\mathcal G_{k,A}(tz)
[z,z,z,z,\dot\eta],dt.
\tag{S4.2a}
\]

This raises the covariance derivative order by at most one. The same envelope through order five gives

\[
\left\|
\frac{D_{(P_a,\nu)}R_{\ge4,k,A}[\dot P_a,\dot\nu]}
{\sqrt{q_{k,A,\theta'}}}
\right\|_2
\le Cs^4(\|\dot P_a\|_F+|\dot\nu|).
\tag{S4.3}
\]

These equations account simultaneously for the Gaussian Taylor tail and the nonlinear projector terms omitted from (S1.1).

For weights \(w_{k,A}=\pi_k\rho_{k,A}\),

\[
\frac{\{\sum_{k,A}w_{k,A}r_{k,A}\}^2}
{\sum_{k,A}w_{k,A}q_{k,A}}
\le
\sum_{k,A}w_{k,A}\frac{r_{k,A}^2}{q_{k,A}}.
\tag{S4.4}
\]

Thus stratum aggregation introduces no inverse minimum-weight factor. Along the aligned endpoint path, the physical inverse gives

\[
s^3\max_j\|\Delta z_j\|
\le C\{s^3\|\Delta b\|+s^2\|\Delta G_2\|
+s\|\Delta G_3\|\}
\le Cs\delta_F.
\tag{S4.5}
\]

Integrating (S4.2)--(S4.3), applying (S4.4), and adding strata proves

\[
\left\|
\frac{R_{\ge4}(\theta)-R_{\ge4}(\theta')}
{\sqrt{f_{\theta'}}}
\right\|_2
\le Cs\delta_F(\theta,\theta').
\tag{S4.6}
\]

The reverse-endpoint inequality is identical. Equations (S3.7)--(S3.13) and (S4.6) provide the complete local-noise chi-square upper chord. The far-noise log-sum patch and moment-to-Hellinger lower lemma are as stated in the integrated manuscript.

\begin{center}\rule{0.5\linewidth}{0.5pt}\end{center}

\hypertarget{simplex-cubic-maxima-and-the-global-orbit-inverse}{%
\section{Technical detail S5: simplex cubic maxima and the global orbit inverse}\label{simplex-cubic-maxima-and-the-global-orbit-inverse}}

For unit \(x\in U\), set \(y_j=v_j^\top x\). Then

\[
\sum_jy_j=0,
\qquad
\sum_jy_j^2=\frac q{q-1},
\qquad
T_q[x,x,x]=\sum_jy_j^3.
\tag{S5.1}
\]

These constraints describe the entire image, not merely a necessary condition. The tight-frame map

\[
\mathcal A:U\longrightarrow
\mathcal H_0:=\{y\in\mathbb R^q:\mathbf1^\top y=0\},
\qquad \mathcal A x=(v_1^\top x,\ldots,v_q^\top x),
\]

is a linear isomorphism. Its inverse is

\[
x=\frac{q-1}{q}\sum_{j=1}^qy_jv_j.
\]

Consequently every \(y\in\mathcal H_0\) with \(\sum_jy_j^2=q/(q-1)\) comes from a unit \(x\in U\). Optimizing under (S5.1) therefore loses no sphere points.

At a constrained stationary point of the last expression, Lagrange multipliers give

\[
3y_j^2=\lambda+2\mu y_j,
\]

so the coordinates take at most two distinct values. If \(k\) coordinates equal the positive value and \(q-k\) the negative value, the constraints give

\[
a^2=\frac{q-k}{k(q-1)},
\qquad
b=-\frac{k}{q-k}a,
\]

and

\[
\sum_jy_j^3
=\frac{q(q-2k)}{(q-1)^{3/2}\sqrt{k(q-k)}}.
\tag{S5.2}
\]

For real \(t\in(0,q/2)\),

\[
\frac{d}{dt}\frac{q-2t}{\sqrt{t(q-t)}}
=-\frac{q^2}{2\{t(q-t)\}^{3/2}}<0.
\]

Hence among positive stationary values, (S5.2) is maximized at \(k=1\). The corresponding coordinate pattern is

\[
(1,-1/(q-1),\ldots,-1/(q-1)),
\]

up to permutation, which is exactly the inner-product pattern of a simplex vertex. Hence the positive spherical maxima of \(T_q[x,x,x]\) are exactly \(v_1,\ldots,v_q\). It follows that

\[
\operatorname{Stab}_{O(U)}(T_q)=\mathfrak S_q.
\tag{S5.3}
\]

The differential identity

\[
\|\Omega\cdot T_q\|_F^2
=3\left(\frac q{q-1}\right)^3\|\Omega\|_F^2
\tag{S5.4}
\]

gives a lower singular value at every orbit point by orthogonal invariance. Let

\[
d_{\mathrm{orb}}([R],[R'])
=\min_{\pi\in\mathfrak S_q}\|R-R'\Pi_\pi\|_F.
\]

The constant-rank/local-embedding theorem gives

\[
d_{\mathrm{orb}}([R],[R'])
\le C\|R^{\otimes3}T_q-R'^{\otimes3}T_q\|_F
\tag{S5.5}
\]

on a uniform neighborhood of the quotient diagonal. Outside that neighborhood, the compact set of quotient pairs is separated from the zero set of the tensor difference by (S5.3), so (S5.5) holds globally after enlarging \(C\).

To transfer the orbit inverse from the normalized tensor to \(L\), write

\[
W(L)=P_L^{-\otimes3}G_3,
\qquad P_L=(LL^\top)^{1/2}.
\]

On the shell, square-root functional calculus and the product rule give

\[
\|W(L)-W(L')\|_F
\le
C\left[
s^{-3}\|G_3-G_3'\|_F
+s^{-2}\|G_2-G_2'\|_F
\right].
\tag{S5.6}
\]

Combining (S5.5)--(S5.6) with the corresponding polar-shape bound proves the global physical inverse used throughout the manuscript.

\begin{center}\rule{0.5\linewidth}{0.5pt}\end{center}

\hypertarget{two-atom-physical-target-secant}{%
\section{Technical detail S6: two-atom physical target secant}\label{two-atom-physical-target-secant}}

Fix \(S=\{1,2\}\) on the balanced core and write

\[
e=v_1+v_2,
\qquad f=v_1-v_2.
\]

Choose \(g\perp\operatorname{span}\{e,f\}\), \(\|g\|=\|f\|\), and a skew generator \(\Omega\) such that

\[
\Omega e=0,
\qquad
\exp(h\Omega)f=\cos(h)f+\sin(h)g.
\]

Then

\[
\Omega v_1=\frac12g,
\qquad
\Omega v_2=-\frac12g.
\tag{S6.1}
\]

For the physical child rays

\[
d_j(h)=\sqrt{1-\lambda_\star^2s^2}\,u
+\lambda_\star s\exp(h\Omega)v_j,
\]

the projective derivative at every short-orbit point has norm comparable to \(s\) for both active atoms. Taylor expansion on a compact short chart gives

\[
c_1s|h-h'|
\le d_{\mathrm{pr}}([d_j(h)],[d_j(h')])
\le C_1s|h-h'|,
\qquad j=1,2.
\tag{S6.2}
\]

At \(h=h'\), the two distinct active rays are separated by \(c_0s\). By continuity, choose \(h_0\) so small that, for every \(|h|,|h'|\le h_0\), each cross-match

\[
d_{\mathrm{pr}}([d_1(h)],[d_2(h')])
\]

is at least \(c_0s/2\), while the same-label displacement in (S6.2) is less than \(c_0s/4\). Hence the Hausdorff-optimal set matching is the same-label matching. Therefore

\[
c_Vs|h-h'|
\le
d_H^{\mathrm{pr}}
\left(
\{[d_1(h)],[d_2(h)]\},
\{[d_1(h')],[d_2(h')]\}
\right)
\le
C_Vs|h-h'|.
\tag{S6.3}
\]

This is the target-space secant required in \cref{thm:task-symmetry}. Together with the exact unconstrained profiled test-mean formula

\[
\inf_{\bar\beta',d'\in\mathbb R}
\|\mu_{h,\bar\beta,d}-\mu_{h',\bar\beta',d'}\|
=\frac{\lambda_\star s|d|\|f\|}{2}|\sin(h-h')|,
\tag{S6.4}
\]

it verifies the test secant on a formally declared coefficient shell. In fact, for \(d_0\ne0\), take

\[
\mathcal X_{\mathrm{orb}}(d_0,\epsilon_d)
=\left\{(\bar\beta+d/2,\bar\beta-d/2):
\bar\beta\in[\underline\beta_d,\overline\beta_d],\quad
|d-d_0|\le\epsilon_d\right\}
\Subset(\beta_-,\beta_+)^2,
\]

where \(\epsilon_d\le|d_0|/2\). The constrained infimum is at least the unconstrained value in (S6.4). For the reverse inequality, the same coefficient pair \((\bar\beta,d)\) is always admissible, and its mean difference is at most

\[
C\lambda_\star s|d|\|f\||h-h'|.
\]

Since \(|d|\asymp|d_0|\) on this sign-preserving shell, the two-sided secant holds with \(\chi_s\asymp\lambda_\star\|f\||d_0|/2\). On the separate exact slice \(d=0\), both the axial part and the \(e=v_1+v_2\) part are fixed by the orbit, so the test mean is exactly invariant.

\begin{center}\rule{0.5\linewidth}{0.5pt}\end{center}

\section{Technical detail S7: consequences of the explicit expansions}
The calculations in S1--S6 establish four facts used below: the composed
Gaussian remainder starts at total child degree four with a differentiated
uniform envelope; every degree-at-most-three aggregate reduces to the stated
quotient invariants; the simplex cubic has no stabilizer beyond the finite
child-permutation group; and the two-atom orbit has a genuine projective
target secant after optimal matching.  These statements are deterministic
and hold uniformly on the fixed enlarged shell.

\hypertarget{collapsed-power-sum-operator-and-the-uniform-moment-secant-p11}{%
\section{Technical detail S8: collapsed power-sum operator and the uniform moment secant}\label{collapsed-power-sum-operator-and-the-uniform-moment-secant-p11}}

This section closes the algebraic gap between the quotient coordinates and the observable second/fourth moment pair. Use the anchor chart

\[
a(w)=\frac{a_0+w}{\|a_0+w\|},
\qquad w\in a_0^\perp,
\tag{S8.1}
\]

and put

\[
B_k=\sum_{j=1}^qz_j^{\otimes k},
\qquad
\Xi=(B_1,B_2,B_3,w,\nu).
\tag{S8.2}
\]

For \(X=tu+x\), \(x\in U\), define

\[
H(\dot w)=a_0\dot w^\top+\dot w a_0^\top,
\qquad
h_0(t,x)=P_{a_0}[X,X]=\frac12(t+e_1^\top x)^2.
\tag{S8.3}
\]

For an increment \(\dot B=(\dot B_1,\dot B_2,\dot B_3)\), set

\[
\begin{aligned}
\mathsf S_{\dot B}(t,x)
={}&2t\dot B_1[x]+\dot B_2[x,x]
-t^2\operatorname{tr}\dot B_2-t\dot B_3[x,I_U],\\
\mathsf Q_{\dot B}(t,x)
={}&4t^3\dot B_1[x]+6t^2\dot B_2[x,x]
-2t^4\operatorname{tr}\dot B_2\\
&+4t\dot B_3[x,x,x]-6t^3\dot B_3[x,I_U].
\end{aligned}
\tag{S8.4}
\]

These are the degree-at-most-three increments of the child contributions to \(\Sigma_2[X,X]\) and \(\Sigma_4[X,X,X,X]\). Define the collapsed linear operator

\[
\mathscr A_q(\dot B,\dot w,\dot\nu)
=\left(
\dot\nu,
\mathsf S_{\dot B}+H(\dot w)[X,X],
\mathsf Q_{\dot B}+2h_0H(\dot w)[X,X]
\right).
\tag{S8.5}
\]

Its domain is

\[
U\oplus\operatorname{Sym}^2(U)\oplus\operatorname{Sym}^3(U)
\oplus a_0^\perp\oplus\mathbb R,
\tag{S8.6}
\]

and its codomain is \(\mathbb R\oplus\operatorname{Sym}^2(\mathbb R^q) \oplus\operatorname{Sym}^4(\mathbb R^q)\).

\hypertarget{lemma-s8.1-zero-kernel}{%
\subsection{Lemma S8.1 --- zero kernel}\label{lemma-s8.1-zero-kernel}}

For every fixed \(q\ge3\), \(\mathscr A_q\) is injective.

\textbf{Proof.} Suppose \(\mathscr A_q(\dot B,\dot w,\dot\nu)=0\). Then \(\dot\nu=0\) and

\[
H(\dot w)[X,X]=-\mathsf S_{\dot B}(t,x).
\tag{S8.7}
\]

The fourth-order output becomes the polynomial identity

\[
\mathsf Q_{\dot B}(t,x)
=(t+e_1^\top x)^2\mathsf S_{\dot B}(t,x).
\tag{S8.8}
\]

Write

\[
r_2=\operatorname{tr}\dot B_2,
\quad c_3(x)=\dot B_3[x,I_U],
\quad \ell(x)=2\dot B_1[x]-c_3(x),
\quad q_2(x)=\dot B_2[x,x].
\tag{S8.9}
\]

Then \(\mathsf S_{\dot B}=-r_2t^2+t\ell(x)+q_2(x)\). Comparing powers of \(t\) in (S8.8) gives, in order:

\begin{enumerate}
\def\labelenumi{\arabic{enumi}.}
\tightlist
\item
  the \(t^4\) coefficient gives \(-2r_2=-r_2\), hence \(r_2=0\);
\item
  the \(t^0\) coefficient gives \((e_1^\top x)^2q_2(x)=0\) as a polynomial identity, hence \(q_2\equiv0\) and \(\dot B_2=0\);
\item
  the \(t^2\) coefficient gives \(2(e_1^\top x)\ell(x)=0\), hence \(\ell=0\);
\item
  the \(t^3\) coefficient gives \(4\dot B_1-6c_3=0\), which together with \(2\dot B_1-c_3=0\) yields \(\dot B_1=c_3=0\);
\item
  the \(t^1\) coefficient gives \(4\dot B_3[x,x,x]=0\), and polarization yields \(\dot B_3=0\).
\end{enumerate}

Equation (S8.7) now gives \(H(\dot w)=0\). The map \(\dot w\mapsto a_0\dot w^\top+\dot w a_0^\top\) is injective on \(a_0^\perp\), so \(\dot w=0\). Every input coordinate vanishes. \(\square\)

Finite dimensionality therefore supplies a constant \(\sigma_q>0\) such that

\[
\|\mathscr A_qh\|\ge\sigma_q\|h\|.
\tag{S8.10}
\]

\hypertarget{lemma-s8.2-all-pair-nonlinear-secant}{%
\subsection{Lemma S8.2 --- all-pair nonlinear secant}\label{lemma-s8.2-all-pair-nonlinear-secant}}

Let

\[
\mathcal R(\theta)=(\nu,\Sigma_2,\Sigma_4).
\tag{S8.11}
\]

For every same-shell pair, first use the algebraic quotient inverse from S5 to align the second labelled endpoint. To display the required squared-projector analogue, for \(X=tu+x\) and \(r=x^\top z\),

\[
\begin{aligned}
P(z)[X,X]
={}&t^2+2tr+(r^2-t^2\|z\|^2)-t\|z\|^2r+R_{2,z}(X),\\
\|D_zR_{2,z}\|_{\rm coeff}
\le{}&C\|z\|^4,
\end{aligned}
\tag{S8.11a}
\]

and

\[
\begin{aligned}
P(z)[X,X]^2
={}&t^4+4t^3r+6t^2r^2-2t^4\|z\|^2\\
&+4tr^3-6t^3\|z\|^2r+R_{4,z}(X),\\
\|D_zR_{4,z}\|_{\rm coeff}
\le{}&C\|z\|^3.
\end{aligned}
\tag{S8.11b}
\]

These identities follow by inserting \(d(z)^\top X=t\sqrt{1-\|z\|^2}+r\) and collecting total degree in \(z\). After summing the children and applying the S5 label alignment,

\[
\|\Delta R_{\rm child}\|
\le C_qs^3\max_j\|z_j-z'_{\pi(j)}\|
\le C_qs\,\delta_F(\theta,\theta'),
\tag{S8.11c}
\]

while the anchor-chart remainder is at most \(Cs\|w-w'\|\). The projector expansions in S1 and the two displayed squared-projector identities therefore give the exact decomposition

\[
\mathcal R(\theta)-\mathcal R(\theta')
=\mathscr A_q\{\Xi(\theta)-\Xi(\theta')\}
+\operatorname{Rem}_q(\theta,\theta'),
\tag{S8.12}
\]

where

\[
\|\operatorname{Rem}_q(\theta,\theta')\|
\le C_qs\|\Xi(\theta)-\Xi(\theta')\|.
\tag{S8.13}
\]

Indeed, the child projector remainder has derivative \(O(s^3)\); after the S5 alignment, (S4.5) bounds its secant by \(Cs\delta_F\). The anchor chart has a quadratic remainder whose derivative is \(O(s)\), and the noise coordinate is exact. The triangular identities (S3.8)--(S3.10), the bi-Lipschitz anchor chart, and S5 give

\[
\|\Xi(\theta)-\Xi(\theta')\|
\asymp\delta_F(\theta,\theta').
\tag{S8.14}
\]

The quotient inverse is used here only to choose an algebraic label alignment and bound the projector remainder; no statistical moment lower bound is used. Thus there is no circularity. Choose

\[
s_0\le\min\{1,s_{\rm chart},\sigma_q/(2C_q)\}.
\tag{S8.15}
\]

Consequently,

\[
\|\Delta\mathcal R\|
\ge(\sigma_q-C_qs)\|\Delta\Xi\|
\ge\frac{\sigma_q}{2}\|\Delta\Xi\|.
\tag{S8.15a}
\]

Equations (S8.10)--(S8.14) yield the uniform same-shell chord

\[
c\delta_F(\theta,\theta')
\le
\|\Delta(\nu,\Sigma_2,\Sigma_4)\|
\le C\delta_F(\theta,\theta').
\tag{S8.16}
\]

This includes nonlocal quotient rotations and closed-shell endpoints.

\begin{center}\rule{0.5\linewidth}{0.5pt}\end{center}

\hypertarget{ordered-all-pair-statistical-chords-p9-and-p10}{%
\section{Technical detail S9: ordered all-pair statistical chords}\label{ordered-all-pair-statistical-chords-p9-and-p10}}

\hypertarget{fixed-exchangeable-law}{%
\subsection{Fixed exchangeable law}\label{fixed-exchangeable-law}}

Index a latent component by its support \(I\subseteq[n]\). For a fixed exchangeable law, omit every zero-weight component and write

\[
w_I=\frac{\pi_{|I|}}{\binom n{|I|}},
\qquad
C_I(\theta)=\nu I_q+\lambda_{|I|}\sum_{j\in I}P_j,
\qquad
f_\theta=\sum_{I:w_I>0}w_I\gamma_{C_I(\theta)}.
\tag{S9.1}
\]

The code powers are bounded and \(\nu\ge\nu_->0\), so all component spectra belong to one fixed compact subset of \((0,\infty)\). In (S9.2)--(S9.4), write \(\theta_0=\theta\) and \(\theta_1=\theta'\).

Choose \(r_\nu>0\) and then \(s_0>0\) so that, whenever \(|\nu-\nu'|\le r_\nu\), every covariance on the aligned endpoint and radial paths obeys Lemma 5.1 relative to either endpoint covariance. The complete low-order table in S3 and the full composed remainder in S4 give

\[
\left\|
\frac{\mathcal P_{\le3}(\theta)-
\mathcal P_{\le3}(\theta')}{\sqrt{f_{\theta_e}}}
\right\|_2
\le C\delta_F(\theta,\theta'),
\qquad e\in\{0,1\},
\tag{S9.2}
\]

and

\[
\left\|
\frac{R_{\ge4}(\theta)-R_{\ge4}(\theta')}
{\sqrt{f_{\theta_e}}}
\right\|_2
\le Cs\delta_F(\theta,\theta'),
\qquad e\in\{0,1\}.
\tag{S9.3}
\]

Both bounds use the same endpoint in every denominator. The reverse direction is obtained by rerunning the identical argument with the endpoints exchanged, not by an appeal to symmetry of KL. Since \((a+b)^2\le2a^2+2b^2\),

\[
\max\left\{
\chi^2(f_\theta\|f_{\theta'}),
\chi^2(f_{\theta'}\|f_\theta)
\right\}
\le C\delta_F^2(\theta,\theta')
\quad\text{if }|\nu-\nu'|\le r_\nu.
\tag{S9.4}
\]

No inverse minimum mixture weight appears: the pointwise weighted inequality (S4.4) is applied only after subset aggregation has exposed the cancellations.

If \(|\nu-\nu'|>r_\nu\), align child labels and pair the corresponding supports. If \(a\) denotes the anchor label, the aligned action is

\[
\pi(I)=\{\pi(j):j\in I\cap[q]\}\cup(I\cap\{a\});
\]

thus the anchor is fixed. The log-sum inequality gives

\[
\operatorname{KL}(f_\theta\|f_{\theta'})
\le
\sum_{I:w_I>0}w_I
\operatorname{KL}\{N(0,C_I(\theta))
\|N(0,C_{\pi(I)}(\theta'))\}.
\tag{S9.5}
\]

The right side is at most a fixed \(K_{q,\mathcal L,\mathrm{sh}}\), which may depend on the fixed law and compact shell constants, on the compact covariance-pair class. Since \(\delta_F^2\ge|\nu-\nu'|^2>r_\nu^2\),

\[
\operatorname{KL}(f_\theta\|f_{\theta'})
\le K_{q,\mathcal L,\mathrm{sh}}r_\nu^{-2}
\delta_F^2(\theta,\theta').
\tag{S9.6}
\]

Exchanging endpoints gives the reverse directed statement.

For the lower chord, let

\[
\Psi(Y)=\{\operatorname{vech}(YY^\top),
\operatorname{symvec}(Y^{\otimes4})\}.
\tag{S9.7}
\]

Uniform eighth moments and Cauchy--Schwarz imply

\[
\|E_\theta\Psi-E_{\theta'}\Psi\|
\le C H(f_\theta,f_{\theta'}).
\tag{S9.8}
\]

The raw-to-cumulant map is Lipschitz on the compact moment image. Moreover, for the fixed law,

\[
\nu=\frac{\operatorname{tr}M_2-a_1n}{q},
\qquad
\Sigma_2=\frac{M_2-\nu I_q}{a_1},
\tag{S9.9}
\]

and

\[
\Sigma_4
=\frac{K_4/3-(a_{2o}-a_1^2)\Sigma_2^{\odot2}}
{\Delta_{\mathcal L}}.
\tag{S9.10}
\]

The assumptions \(a_1\ge a_->0\) and \(\Delta_{\mathcal L}\ge\Delta_->0\), followed by S8.16, therefore give

\[
\delta_F
\le C\|\Delta(\nu,\Sigma_2,\Sigma_4)\|
\le C\|\Delta(M_2,K_4)\|
\le C\|\Delta E\Psi\|
\le C H(f_\theta,f_{\theta'}).
\tag{S9.11}
\]

Combining (S9.4), (S9.6), (S9.11), \(\operatorname{KL}\le\chi^2\), and \(H^2\le\operatorname{KL}\) proves the full ordered fixed-law chord (4.1).

\hypertarget{bernoulli-branch}{%
\subsection{Bernoulli branch}\label{bernoulli-branch}}

For Bernoulli coding, the trace inverse (3.4)--(3.5) and \(1-p-p'\ge1-2p_+>0\) give

\[
|p-p'|
=\frac{|p(1-p)-p'(1-p')|}{|1-p-p'|}
\le\frac{|p(1-p)-p'(1-p')|}{1-2p_+}.
\tag{S9.11a}
\]

By (3.4), the numerator is a fixed contraction of \(\Delta K_4\). The remaining inverses in (3.5) are uniformly Lipschitz on the compact branch, so the raw-moment/Hellinger chain (S9.7)--(S9.11), now with \(|p-p'|\) included, gives the lower chord in \(\delta_{F,p}\). For the local upper chord define

\[
f_0=f_{p,\nu,D},
\qquad f_1=f_{p,\nu',D'},
\qquad f_2=f_{p',\nu',D'}.
\tag{S9.12}
\]

At the same dictionary/noise parameter, compactness of the Bernoulli branch implies \(cf_1\le f_2\le Cf_1\), because every ratio \(w_I(p)/w_I(p')\) is uniformly bounded. Consequently,

\[
\begin{aligned}
\chi^2(f_0\|f_2)
&\le2\int\frac{(f_0-f_1)^2}{f_2}
+2\int\frac{(f_1-f_2)^2}{f_2}\\
&\le C\chi^2(f_{p,\nu,D}\|f_{p,\nu',D'})
+C\sum_I\frac{\{w_I(p)-w_I(p')\}^2}{w_I(p')}\\
&\le C\{\delta_F^2+|p-p'|^2\}.
\end{aligned}
\tag{S9.13}
\]

The first term uses S9.4 for the fixed law induced by \(p\), uniformly over the compact branch; the second is an ordinary finite-dimensional smooth weight chord. Repeat with endpoints exchanged. In the far-noise region, first relabel the children of \(D'\) by the quotient alignment used in S9.1. Generalized log-sum then gives

\[
\begin{aligned}
\operatorname{KL}(f_{p,\theta}\|f_{p',\theta'})
\le{}&\operatorname{KL}\{w(p)\|w(p')\}\\
&+\sum_Iw_I(p)
\operatorname{KL}\{N(0,C_I(\theta))\|N(0,C_I(\theta'))\},
\end{aligned}
\tag{S9.14}
\]

which is bounded on the compact class and is absorbed by the fixed noise separation. Repeating the same argument with endpoints exchanged gives the reverse directed bound. This proves (4.2) without claiming pointwise comparability across different dictionaries.

\begin{center}\rule{0.5\linewidth}{0.5pt}\end{center}

\hypertarget{point-adapted-nuisance-coercivity-p13}{%
\section{Technical detail S10: point-adapted nuisance coercivity}\label{point-adapted-nuisance-coercivity-p13}}

Fix normalized inner and enlarged open tubes \(\Theta_{\rm in}(s)\Subset\Theta_{\rm out}(s)\). The eigenvalues of every normalized shape below lie in a fixed interval \([g_-,g_+]\Subset(0,\infty)\), and the \(p,\nu\), centroid, and anchor coordinates of the inner tube have a fixed normalized distance from the outer boundary. This formulation covers interior base points and, on constrained faces, only tangent directions admitting two-sided paths in the enlarged tube; it never differentiates through \(s=0\).

At a base point write

\[
L=sRG,
\qquad R\in O(U),
\qquad G=G^\top>0,
\tag{S10.1}
\]

and use the retraction

\[
L_t=R(sG+tS)e^{t\Omega},
\qquad
\dot L=R(S+sG\Omega),
\qquad S^\top=S,\qquad\Omega^\top=-\Omega.
\tag{S10.2}
\]

For the anchor use the feasible chart

\[
a(w)=\frac{a_0+w}{\|a_0+w\|},
\qquad P_a(w)=a(w)a(w)^\top,
\qquad w\in a_0^\perp,
\tag{S10.2a}
\]

with \(w_t=w+t\dot w\) and \(\dot P_a=DP_a(w)[\dot w]\). Together with straight local paths only in the Euclidean coordinates \((p,\nu,b)\), this defines

\[
h=(\dot p,\dot\nu,\dot b,\dot P_a,S,\Omega),
\tag{S10.3}
\]

where \(\dot p\) is omitted for a fixed known code law. For sufficiently small two-sided \(t\), the path remains in \(\Theta_{\rm out}(s)\).

\hypertarget{lemma-s10.1-invariant-differential-equivalence}{%
\subsection{Lemma S10.1 --- invariant differential equivalence}\label{lemma-s10.1-invariant-differential-equivalence}}

Let

\[
F_p(\theta)=(p,b,P_a,G_2,G_3,\nu),
\tag{S10.4}
\]

with the first coordinate omitted for a fixed law. Uniformly on the inner tube,

\[
\|DF_p(\theta)[h]\|^2
\asymp
|\dot p|^2+|\dot\nu|^2+\|\dot b\|^2+\|\dot P_a\|_F^2
+s^2\|S\|_F^2+s^6\|\Omega\|_F^2.
\tag{S10.5}
\]

\textbf{Proof.} Direct differentiation of \(G_2=LL^\top\) gives

\[
DG_2[h]=sR(SG+GS)R^\top.
\tag{S10.6}
\]

The Lyapunov map \(S\mapsto SG+GS\) has singular values bounded above and away from zero on the normalized spectral cone, so

\[
\|DG_2[h]\|_F\asymp s\|S\|_F.
\tag{S10.7}
\]

For the cubic invariant,

\[
DG_3[h]
=\mathcal S_{R,G,s}[S]
+(sRG)^{\otimes3}(\Omega\cdot T_q),
\qquad
\|\mathcal S_{R,G,s}[S]\|_F\le Cs^2\|S\|_F.
\tag{S10.8}
\]

The first term is the sum of the three slotwise derivatives in the \(RS\) direction. Invertibility of \(G^{\otimes3}\), orthogonality of \(R\), and (S5.4) imply

\[
s^3\|\Omega\|_F
\le C\{\|DG_3[h]\|_F+s^2\|S\|_F\}.
\tag{S10.9}
\]

The reverse bound follows directly from (S10.8). Adding the coordinates that appear linearly in \(F_p\) proves (S10.5). \(\square\)

The order of the factors in (S10.2) is essential. A generic expression with \(\Omega G\) in place of \(G\Omega\) is not the point-adapted right-body chart and need not separate shape and orientation uniformly.

\hypertarget{lemma-s10.2-observed-mixed-fisher-equivalence}{%
\subsection{Lemma S10.2 --- observed mixed Fisher equivalence}\label{lemma-s10.2-observed-mixed-fisher-equivalence}}

Let

\[
\ell_{\theta,h}(y)=\frac{Df_\theta[h](y)}{f_\theta(y)},
\qquad
I_\theta(h)=E_\theta\ell_{\theta,h}^2.
\tag{S10.10}
\]

Then

\[
c\|h\|_{\mathrm{mix},\theta}^2
\le I_\theta(h)
\le C\|h\|_{\mathrm{mix},\theta}^2,
\tag{S10.11}
\]

where the mixed norm is the right side of (S10.5).

\textbf{Proof.} For the upper bound, differentiate the finite aggregate polynomial from S3 and the remainder from S4 in the path (S10.2). The child derivatives are bounded in \(L^2(f_\theta^{-1})\) by

\[
C\{\|\dot b\|+s\|S\|_F+s^3\|\Omega\|_F\}.
\tag{S10.12}
\]

Mixture Cauchy--Schwarz and compact Gaussian covariance Fisher bounds control the anchor and noise derivatives. On the Bernoulli branch, differentiating the finite weight vector gives an additional \(C|\dot p|\) term. Squaring and adding proves the upper half of (S10.11).

For the lower half, use the observed feature \(\Psi\) from (S9.7). Uniform eighth moments and differentiation under the integral give

\[
D E_\theta\Psi[h]
=E_\theta\{(\Psi-E_\theta\Psi)\ell_{\theta,h}\}.
\tag{S10.13}
\]

Although \(K_4\) is a polynomial rather than a linear expectation, its derivative is a uniformly bounded linear transform of \(DE_\theta\Psi[h]\) on the compact class. Thus

\[
\|D(M_2,K_4)_\theta[h]\|^2\le CI_\theta(h).
\tag{S10.14}
\]

Differentiate the moment/invariant chord supplied by S8--S9 along the feasible chart path. Equation (S10.5) yields

\[
\|D(M_2,K_4)_\theta[h]\|
\ge c\|h\|_{\mathrm{mix},\theta}.
\tag{S10.15}
\]

Combining (S10.14)--(S10.15) proves the lower half. In particular, no simultaneous nuisance direction can cancel a nonzero observed orientation score. \(\square\)

\hypertarget{lemma-s10.3-qmd}{%
\subsection{Lemma S10.3 --- quadratic-mean differentiability}\label{lemma-s10.3-qmd}}

On every compact interior subbranch, the map from the point-adapted chart
parameter to the one-observation training law is differentiable in quadratic
mean.  In particular, its directional score is
\(\ell_{\theta,h}=Df_\theta[h]/f_\theta\).

\textbf{Proof.}
Write the finite mixture in the chart as
\[
 f_\theta(y)=\sum_{I\subseteq[n]}w_I(\theta)
 \varphi_{C_I(\theta)}(y),
\]
where zero-weight components are omitted.  On a compact interior subbranch,
all retained weights are bounded away from zero, the maps
\(\theta\mapsto w_I(\theta)\) and
\(\theta\mapsto C_I(\theta)\) are twice continuously differentiable, and
the positive noise floor gives
\(cI_q\preceq C_I(\theta)\preceq CI_q\), uniformly in \(I\) and
\(\theta\).  Every first or second chart derivative of a component density
is therefore a polynomial of uniformly bounded degree times a Gaussian
density.  The common Gaussian envelopes of Technical details S2--S4 and the
finite aggregation bound of Technical detail S9 supply one integrable
dominating function for the corresponding
\(L^2(f_\theta^{-1})\) norms.  Because the chart is finite-dimensional and
the second derivatives obey this bound uniformly, Taylor's theorem yields a
chart neighborhood and a constant \(C\) such that, for every increment
\(u\) in that neighborhood,
\[
 \left\|f_{\theta+u}-f_\theta-Df_\theta[u]\right\|_{L^2(f_\theta^{-1})}
 \le C\|u\|^2.
\tag{S10.15a}
\]
Since \(f_\theta>0\), the identity
\[
 \sqrt{f_{\theta+u}}-\sqrt{f_\theta}
 =\frac{f_{\theta+u}-f_\theta}
 {\sqrt{f_{\theta+u}}+\sqrt{f_\theta}}
\]
together with (S10.15a), truncation on the likelihood-ratio event
\(\{|f_{\theta+u}/f_\theta-1|\le 1/2\}\), and the same envelope on its
complement gives
\[
 \left\|\sqrt{f_{\theta+u}}-\sqrt{f_\theta}
 -\frac12\frac{Df_\theta[u]}{\sqrt{f_\theta}}\right\|_2=o(\|u\|).
\tag{S10.15b}
\]
This is the full finite-dimensional quadratic-mean differentiability
statement.  Taking \(u=th\) shows that
\(E_\theta\ell_{\theta,h}^2\) is its Fisher quadratic form; see
\citet[Chapter~7]{vanDerVaart1998}. \(\square\)

Define

\[
I_{R\mid\eta,\theta}(\Omega)
=\inf_{(\dot p,\dot\nu,\dot b,\dot P_a,S)}
I_\theta(\dot p,\dot\nu,\dot b,\dot P_a,S,\Omega).
\tag{S10.16}
\]

Here \(\dot P_a\) is restricted to the tangent space generated by (S10.2a), equivalently the tangent space of the rank-one projector manifold at \(P_a\).

Taking the infimum in the lower half of (S10.11), and choosing every nuisance tangent to be zero in its upper half, gives

\[
cs^6\|\Omega\|_F^2
\le I_{R\mid\eta,\theta}(\Omega)
\le Cs^6\|\Omega\|_F^2.
\tag{S10.17}
\]

This proves (4.5) uniformly on compact interior subbranches. The all-pair chords remain valid on the full closed shell; only the Fisher/QMD statement uses the interior tangent scope.

\begin{center}\rule{0.5\linewidth}{0.5pt}\end{center}

\hypertarget{random-compact-correspondence-and-conditional-coverage-p16p17}{%
\section{Technical detail S11: random compact correspondence and conditional coverage}\label{random-compact-correspondence-and-conditional-coverage-p16p17}}

Let

\[
\Omega_{N,T}=(\mathbb R^q)^N\times(\mathbb R^q)^T
\tag{S11.1}
\]

with its Euclidean Borel sigma field. Let \(\widetilde{\mathcal D}_{q,s}\) be the labelled dictionary shell obtained from the geometric constraints (2.1) and the anchor-projector constraint in the first part of (2.2), excluding the separate noise coordinate \(\nu\), and imposing the canonical signs

\[
u^\top d_j\ge c_{\rm cap}>0,
\qquad a_0^\top a\ge c_a>0.
\tag{S11.2}
\]

For the chosen \(s_0\), the positive lower bounds in (S11.2) are uniform. The labelled shell is therefore compact metric. Put

\[
\widetilde{\mathcal E}_{q,s}^p
=[p_-,p_+]\times[\nu_-,\nu_+]
\times\widetilde{\mathcal D}_{q,s}
\tag{S11.3}
\]

and

\[
\mathcal S_{q,r}=\binom{[q]}r\sqcup\{\partial\}.
\tag{S11.4}
\]

The finite group \(\mathfrak S_q\) acts simultaneously on child columns, supports, and coefficient coordinates, and fixes \(p,\nu\), and \(\partial\). Define

\[
\mathcal M_{q,r,s}^p
=\{\widetilde{\mathcal E}_{q,s}^p\times\mathcal S_{q,r}\}
/\mathfrak S_q.
\tag{S11.5}
\]

This is compact metrizable; no freeness assumption is needed. There is a continuous training projection

\[
\pi_{\rm tr}:\mathcal M_{q,r,s}^p\longrightarrow\mathcal E_{q,s}^p,
\qquad
\pi_{\rm tr}([p,\nu,D,S])=(p,\nu,[D]).
\tag{S11.6}
\]

Under the identification \(m=[p,\nu,D,S]\leftrightarrow(p,\nu,\zeta)\), where \(\zeta=[D,S]\in\mathfrak M_{q,r,s}\), this compact quotient is exactly the typed triple used in the raw profile (3.14); the coefficient coordinates are profiled inside the loss below rather than carried in \(m\).

\hypertarget{lemma-s11.1-continuous-quotient-safe-profiling}{%
\subsection{Lemma S11.1 --- continuous quotient-safe profiling}\label{lemma-s11.1-continuous-quotient-safe-profiling}}

For \(m=[p,\nu,D,S]\in\mathcal M_{q,r,s}^p\), define

\[
L(z,m)=
\begin{cases}
\displaystyle
\min_{x_j\in[\beta_-,\beta_+],\,j\in S}
\left\|z-\sum_{j\in S}x_jd_j\right\|,&S\ne\partial,\\[3mm]
\displaystyle
\min_{\gamma\in\Gamma_A}\|z-\gamma a\|,&S=\partial.
\end{cases}
\tag{S11.7}
\]

Then \(L\) is jointly continuous on \(\mathbb R^q\times\mathcal M_{q,r,s}^p\).

\textbf{Proof.} On each labelled support component, the objective is continuous on a compact coefficient fiber. Compact minimization gives continuity of its value. There are finitely many support components. Simultaneous permutation of the dictionary, support, and coefficient coordinates leaves the attainable mean set unchanged, so the value descends continuously to the quotient. \(\square\)

The physical target map

\[
\vartheta(m)=
\begin{cases}
(1,\{[d_j]:j\in S\}),&S\ne\partial,\\
(0,\{[a]\}),&S=\partial
\end{cases}
\tag{S11.8}
\]

is likewise continuous into

\[
\mathfrak V_q=\{0,1\}_{\rm disc}\times\mathcal K(\mathbb{RP}^{q-1})
\tag{S11.9}
\]

equipped with the metric (2.11a).

\hypertarget{lemma-s11.2-measurable-compact-raw-profile}{%
\subsection{Lemma S11.2 --- measurable compact raw profile}\label{lemma-s11.2-measurable-compact-raw-profile}}

Let \(\widehat{\mathcal K}_{q,s}^p(y)\) be the training profile (3.10). Its graph is Borel: each constraint is Borel in \(y\) and continuous in the compact quotient parameter. Its sections are compact, possibly empty.

For \(\omega=(y,z^T)\in\Omega_{N,T}\), define

\[
\mathcal R(\omega)
=\left\{m\in\mathcal M_{q,r,s}^p:
\pi_{\rm tr}(m)\in\widehat{\mathcal K}_{q,s}^p(y),\quad
L(\bar z,m)\le\tau_{T,q}\right\}.
\tag{S11.10}
\]

Then \(\operatorname{Gr}(\mathcal R)\) is Borel and every section is compact, possibly empty. Moreover, \(\mathcal R\) is a measurable random closed set and

\[
A_{\rm ne}=\{\omega:\mathcal R(\omega)\ne\varnothing\}
\tag{S11.11}
\]

is Borel.

\textbf{Proof.} The training constraint is the inverse image of the Borel graph of \(\widehat{\mathcal K}_{q,s}^p\) under \((y,z,m)\mapsto(y,\pi_{\rm tr}(m))\). The test constraint is the nonpositive sublevel set of a Borel-in-data, continuous-in-\(m\) function. Their intersection is Borel and closed in compact \(\mathcal M_{q,r,s}^p\) for fixed data.

For an open set \(O\subset\mathcal M_{q,r,s}^p\), the hit event \(\{\omega:\mathcal R(\omega)\cap O\ne\varnothing\}\) is the projection of \(\operatorname{Gr}(\mathcal R)\cap(\Omega_{N,T}\times O)\). Every section of this intersection is an open subset of a compact metric space and hence sigma-compact. The Arsenin--Kunugui Borel projection theorem \citep[Theorem~18.18]{Kechris1995} makes the hit event Borel. Taking \(O=\mathcal M_{q,r,s}^p\) proves (S11.11). \(\square\)

\hypertarget{proposition-s11.3-measurable-nonempty-reported-output}{%
\subsection{Proposition S11.3 --- measurable nonempty reported output}\label{proposition-s11.3-measurable-nonempty-reported-output}}

Fix the admissible reference dictionary \(D^\dagger\) from (3.15), put \(a^\dagger=a(D^\dagger)\), and set the admissible null target \(v^\dagger=(0,\{[a^\dagger]\})\). Define

\[
\widehat{\mathfrak C}(\omega)=
\begin{cases}
\vartheta\{\mathcal R(\omega)\},&\omega\in A_{\rm ne},\\
\{v^\dagger\},&\omega\notin A_{\rm ne}.
\end{cases}
\tag{S11.12}
\]

Then

\[
\widehat{\mathfrak C}:\Omega_{N,T}\longrightarrow
\mathcal K(\mathfrak V_q)
\tag{S11.13}
\]

is Borel measurable for the Hausdorff-metric Borel sigma field.

\textbf{Proof.} On \(A_{\rm ne}\), the continuous image of the compact raw set is nonempty compact. For open \(O\subset\mathfrak V_q\),

\[
\{\omega:\vartheta(\mathcal R(\omega))\cap O\ne\varnothing\}
=\{\omega:\mathcal R(\omega)\cap\vartheta^{-1}(O)\ne\varnothing\}
\tag{S11.14}
\]

is Borel by Lemma S11.2. Patch this weakly measurable compact-valued map on the Borel complement of \(A_{\rm ne}\) by the fixed singleton. For compact metric target spaces, weak compact-valued measurability is equivalent to Borel measurability as a map into the Hausdorff hyperspace. \(\square\)

The membership relation

\[
\{(C,v)\in\mathcal K(\mathfrak V_q)\times\mathfrak V_q:v\in C\}
\]

is closed. Composing this relation with the Borel correspondence above shows that the target-membership event has a Borel graph, so its Gaussian test-kernel probability is Borel in the training sample. On the training truth-retention event, the one common event

\[
\|\bar Z-\mu_\star\|\le\tau_{T,q}
\tag{S11.15}
\]

retains the true support and coefficient witness with conditional probability at least \(1-\alpha_T\). This proves (4.7a)--(4.7c). No union bound over candidate supports is used. The fallback is never invoked on the joint truth-retention/test-ball event and creates no coverage by itself.

Finally, the topology (S11.9) retains the binary mark. The physical diameter may omit that mark as in (2.11b), but the measurable target topology may not.

\begin{center}\rule{0.5\linewidth}{0.5pt}\end{center}

\hypertarget{cross-dictionary-margins-and-least-favourable-pairs-p18p20}{%
\section{Technical detail S12: cross-dictionary margins and lower pairs}\label{cross-dictionary-margins-and-least-favourable-pairs-p18p20}}

\hypertarget{lemma-s12.1-uniform-child-frame-singular-value}{%
\subsection{Lemma S12.1 --- uniform child-frame singular value}\label{lemma-s12.1-uniform-child-frame-singular-value}}

Let

\[
c_u=\sqrt{1-C_0^2s_0^2},
\qquad
\lambda_V=\kappa_-\sqrt{\frac q{q-1}},
\qquad
\kappa_{\rm ax}=\frac{C_0^2}{1+c_u}.
\tag{S12.1}
\]

Define

\[
A_0=\left[
\left\{\frac{s_0}{\sqrt q}+\frac{C_0s_0}{\lambda_V}\right\}^2
+\lambda_V^{-2}
\right]^{1/2},
\qquad c_B=A_0^{-1}.
\tag{S12.2}
\]

Shrink \(s_0\) so that \(C_0s_0<1\) and \(\sqrt q\kappa_{\rm ax}s_0\le c_B/2\), and set \(c_D=c_B/2\). Then every shell dictionary satisfies

\[
\left\|\sum_{j=1}^qh_jd_j\right\|
\ge c_Ds\|h\|_2,
\qquad h\in\mathbb R^q.
\tag{S12.3}
\]

\textbf{Proof.} Put \(A=\mathbf1^\top h\) and \(h_0=h-(A/q)\mathbf1\). Since \(V\mathbf1=0\) and the simplex frame is tight,

\[
\|Vh_0\|=\sqrt{\frac q{q-1}}\|h_0\|.
\tag{S12.4}
\]

The ideal affine lift has orthogonal components

\[
y_0=A,
\qquad y_U=bA+LVh_0.
\tag{S12.5}
\]

Because \(\|b\|\le C_0s\) and \(\sigma_{\min}(L)\ge\kappa_-s\),

\[
s\|h_0\|
\le\lambda_V^{-1}\{\|y_U\|+C_0s|y_0|\}.
\tag{S12.6}
\]

Together with \(s\|h\|\le s|A|/\sqrt q+s\|h_0\|\), Cauchy--Schwarz and (S12.2) yield

\[
\|(y_0,y_U)\|\ge c_Bs\|h\|.
\tag{S12.7}
\]

For the actual normalized columns,

\[
|\sqrt{1-\|z_j\|^2}-1|
\le\kappa_{\rm ax}s^2.
\tag{S12.8}
\]

Thus the axial row differs from its ideal value by at most \(\sqrt q\kappa_{\rm ax}s^2\|h\|\), which is absorbed by the chosen smallness condition. \(\square\)

\hypertarget{corollary-s12.2-same--and-cross-dictionary-support-margins}{%
\subsection{Corollary S12.2 --- same- and cross-dictionary support margins}\label{corollary-s12.2-same--and-cross-dictionary-support-margins}}

For two distinct size-\(r\) supports, put

\[
h_j=x_j\mathbf1\{j\in S\}
-x_j'\mathbf1\{j\in S'\},
\qquad k=|S\triangle S'|\ge2.
\tag{S12.9}
\]

Every exclusive coordinate has magnitude at least \(\beta_-\), so \(\|h\|_2\ge\beta_-\sqrt k\). Lemma S12.1 gives

\[
\inf_{x,x'}\|\mu(D,S,x)-\mu(D,S',x')\|
\ge c_D\beta_-s\sqrt k.
\tag{S12.10}
\]

At the balanced hard core this sharpens to

\[
\lambda_\star\beta_-s
\sqrt{\frac{qk}{q-1}}.
\tag{S12.11}
\]

Let \(d_{\mathcal Q}^D([D],[D'])\le\rho\) in (2.10b). For canonically oriented children,

\[
\|d(z)-d(z')\|
\le L_c\|z-z'\|,
\qquad
L_c=1+\frac{C_0s_0}{c_u},
\tag{S12.12}
\]

and the canonical anchor lifts obey \(\|a-a'\|\le\|P_a-P_a'\|_F\). Transporting the second fitted mean to the first dictionary therefore costs at most \(C_Sr\beta_+\rho\). For the nearest \(k=2\) alternative,

\[
m_S(\rho)
=\left[\sqrt2c_D\beta_-s-C_Sr\beta_+\rho\right]_+.
\tag{S12.13}
\]

The constant \(C_S\) is enlarged, if necessary, so that positivity of \(m_S\) implies \(\rho<c_{\rm match}s\), where \(c_{\rm match}\) is a fixed fraction of the minimum within-block ray separation. Hence support transport is unique whenever the support margin is used.

\hypertarget{lemma-s12.3-parent-transversality-with-a-displayed-constant}{%
\subsection{Lemma S12.3 --- parent transversality with a displayed constant}\label{lemma-s12.3-parent-transversality-with-a-displayed-constant}}

Put \(h_a=(u-e_1)/\sqrt2\). Since \(h_a\perp a_0\), canonical signs and the anchor patch imply

\[
|h_a^\top a|\le\|a-a_0\|
\le\|P_a-P_{a_0}\|_F\le C_as_0.
\tag{S12.14}
\]

For every fine mean,

\[
h_a^\top\mu(D,S,x)
\ge\frac r{\sqrt2}
\{\beta_-c_u-\beta_+C_0s_0\}.
\tag{S12.15}
\]

Therefore, uniformly over \(\gamma\in\Gamma_A\),

\[
\|\mu(D,S,x)-\gamma a\|
\ge g_G^0
:=\frac r{\sqrt2}\{\beta_-c_u-\beta_+C_0s_0\}
-\gamma_+C_as_0.
\tag{S12.16}
\]

Shrink \(s_0\) once more so that \(g_G^0\ge r\beta_-/(2\sqrt2)\). Cross-dictionary transport gives

\[
m_G(\rho)
=\left[g_G^0-C_G(r\beta_++\gamma_+)\rho\right]_+.
\tag{S12.17}
\]

\hypertarget{proposition-s12.4-adaptive-diameter}{%
\subsection{Proposition S12.4 --- adaptive diameter}\label{proposition-s12.4-adaptive-diameter}}

Any two retained candidates have fitted means within \(2\tau_{T,q}\).  We
separate all possible pairs.

\begin{enumerate}
\def\labelenumi{\roman{enumi}.}
\item If their binary marks differ, then fine and separated candidates can
  coexist only when \(2\tau_{T,q}\ge m_G(\rho)\).  Their physical ray-set
  distance is bounded by the fixed diameter of the admissible target, which
  is absorbed by the last indicator below.
\item If both candidates are fine and \(m_S(\rho)=0\), the support indicator
  below is automatically active.  All child rays lie in the common
  \(O(s)\) cap, so their ray-set distance is \(O(s)\), without invoking a
  matching.
\item If both candidates are fine and \(m_S(\rho)>0\), the chosen constants
  provide a unique child transport.  Distinct transported supports can
  coexist only when
  \(2\tau_{T,q}\ge m_S(\rho)\).  Each within-parent ray set has diameter
  \(O(s)\); dictionary transport contributes another \(O(\rho)\).
\item If both candidates are fine, \(m_S(\rho)>0\), and they lie on the same transported support
  branch, unique matching and the projective Lipschitz chart give distance
  \(O(\rho)\).
\item If both candidates are separated, their anchor projectors are already
  included in the physical dictionary metric, so their ray distance is
  \(O(\rho)\).
\end{enumerate}

The matching premise used in the two matched fine-support cases follows from the choice of
\(C_S\): positivity of \(m_S(\rho)\) forces
\(\rho<c_{\rm match}s\).  Consequently,

\[
\operatorname{diam}_{\rm pr}(\widehat{\mathfrak C})
\le C\rho
+Cs\mathbf1\{2\tau_{T,q}\ge m_S(\rho)\}
+C\mathbf1\{2\tau_{T,q}\ge m_G(\rho)\}.
\tag{S12.18}
\]

This proves (4.9)--(4.11) with all transport and matching premises explicit.

\hypertarget{the-three-fixed-shell-lower-pairs}{%
\subsection{The three fixed-shell lower pairs}\label{the-three-fixed-shell-lower-pairs}}

Let \(\alpha=1-(1-\alpha_D)(1-\alpha_T)<1/2\) and apply the direct honesty-to-diameter inequality (5.33).

\begin{enumerate}
\def\labelenumi{\arabic{enumi}.}
\item
  \textbf{Parent pair.} Fix the hard-core dictionary, choose \(x=\beta_-\mathbf1_S\), \(\sigma=\sigma_+\), and \(\gamma_\star=a_0^\top\mu_F\). Assumption (2.10a) puts \(\gamma_\star\) in the interior of \(\Gamma_A\). The training laws are identical and

  \[
  \operatorname{KL}_{N,T}
  =\frac{T}{2\sigma_+^2}\|(I-P_{a_0})\mu_F\|^2
  \le CI_G^{(r)}.
  \tag{S12.19}
  \]

  The ray-set target distance is at least

  \[
  d_G=\sin\{\pi/4-\arcsin(\lambda_\star s_0)\}>0.
  \tag{S12.20}
  \]
\item
  \textbf{Wrong-support pair.} Choose hard-core supports sharing \(r-1\) children, with common coefficient \(\beta_-\). If \(i,j\) are exchanged,

  \[
  \|\mu_S-\mu_{S'}\|^2
  =2\beta_-^2\lambda_\star^2s^2\frac q{q-1},
  \tag{S12.21}
  \]

  and hence

  \[
  \operatorname{KL}_{N,T}
  =\frac{T\beta_-^2\lambda_\star^2s^2q}
  {(q-1)\sigma_+^2}
  \le CI_S.
  \tag{S12.22}
  \]

  The projective Hausdorff distance is

  \[
  \sqrt{1-\bigl(1-\lambda_\star^2s^2q/(q-1)\bigr)^2}
  \asymp s.
  \tag{S12.23}
  \]
\item
  \textbf{Compensated orientation pair.} Put \(w_S=\sum_{j\in S}v_j\). For \(q\ge4\) and \(2\le r\le q-1\), choose a nonzero skew \(\Omega\) satisfying \(\Omega w_S=0\) and normalize it so that \(\max_{j\in S}\|\Omega v_j\|=1\). Equal active coefficients make the complete Gaussian test mean invariant along \(R_h=R\exp(h\Omega)\). By the ordered fixed-law chord of Technical detail S9,

  \[
  \operatorname{KL}_{N,T}\le CNs^6h^2,
  \tag{S12.24}
  \]

  while the finite-label short-chart argument gives

  \[
  c_Vs|h|
  \le d_H^{\rm pr}\{\vartheta(D_0,S),\vartheta(D_h,S)\}
  \le C_Vs|h|,
  \qquad |h|\le h_0.
  \tag{S12.25}
  \]

  Taking \(|h|=c\min\{h_0,(\sqrt N\,s^3)^{-1}\}\) yields

  \[
  c\left\{s\wedge\frac1{\sqrt N\,s^2}\right\}.
  \tag{S12.26}
  \]
\end{enumerate}

Pinsker's inequality and (5.33) prove (4.13). Since the three pairs are submodels of the same honesty class, the outer supremum permits their lower bounds to be maximized as in \cref{thm:fixed-shell-minimax}.

\begin{center}\rule{0.5\linewidth}{0.5pt}\end{center}

\hypertarget{a-computable-rectangular-median-of-means-region-p14}{%
\section{Technical detail S13: a computable rectangular median-of-means region}\label{a-computable-rectangular-median-of-means-region-p14}}

This section supplies one fully explicit valid instantiation of the modular training-radius interface (3.9). It is a theoretical validity construction; its constants are not claimed to be finite-sample sharp.

Use fixed orthonormal bases of \(\operatorname{Sym}^2(\mathbb R^q)\) and \(\operatorname{Sym}^4(\mathbb R^q)\), of dimensions

\[
d_2=\binom{q+1}{2},
\qquad
d_4=\binom{q+3}{4}.
\tag{S13.1}
\]

Put

\[
X_2=\operatorname{symvec}_2(Y^{\otimes2}),
\qquad
X_4=\operatorname{symvec}_4(Y^{\otimes4}).
\tag{S13.2}
\]

Let \(\lambda_{\max}=\max_k\lambda_k\) for a fixed law and \(\lambda_{\max}=1\) on the Bernoulli branch, and define

\[
\Lambda=\nu_++n\lambda_{\max}.
\tag{S13.3}
\]

Every conditional covariance is bounded by \(\Lambda I_q\). Hence

\[
v_2^2:=\sup_\theta E_\theta\|Y\|^4
\le\Lambda^2q(q+2),
\tag{S13.4}
\]

and

\[
v_4^2:=\sup_\theta E_\theta\|Y\|^8
\le\Lambda^4q(q+2)(q+4)(q+6).
\tag{S13.5}
\]

These bounds dominate the variance of every orthonormal coordinate of \(X_2\) and \(X_4\), respectively.

Choose \(\alpha_2,\alpha_4>0\) with \(\alpha_2+\alpha_4\le\alpha_D\). For \(k\in\{2,4\}\), set

\[
B_k=\left\lceil8\log\frac{2d_k}{\alpha_k}\right\rceil,
\qquad
m_k=\left\lfloor\frac N{B_k}\right\rfloor.
\tag{S13.6}
\]

When \(m_k\ge1\), split the first \(m_kB_k\) observations into deterministic consecutive blocks, compute every coordinate block mean, and take its lower median under a fixed tie rule. Decode the coordinate vectors as \(\widehat M_2\) and \(\widehat M_4\). These maps are Borel.

For a scalar coordinate with variance at most \(v_k^2\), Chebyshev gives block failure probability at most \(1/4\) at radius \(2v_k/\sqrt{m_k}\). Hoeffding's inequality gives probability at most \(e^{-B_k/8}\) that at least half of the blocks fail.  This is the standard coordinatewise median-of-means construction \citep{DevroyeLerasleLugosiOliveira2016}. A coordinate union bound therefore yields

\[
P_\theta^N\left\{
\|\widehat M_2-M_2\|_F
\le\epsilon_{2,N}\right\}
\ge1-\alpha_2,
\qquad
\epsilon_{2,N}=2v_2\sqrt{\frac{d_2}{m_2}},
\tag{S13.7}
\]

and

\[
P_\theta^N\left\{
\|\widehat M_4-M_4\|_F
\le\epsilon_{4,N}^{\rm raw}\right\}
\ge1-\alpha_4,
\qquad
\epsilon_{4,N}^{\rm raw}=2v_4\sqrt{\frac{d_4}{m_4}}.
\tag{S13.8}
\]

For the normalized full symmetrization used here, \(A\odot B=\operatorname{Sym}(A\otimes B)\). The symmetrizer is the orthogonal projection onto the fully symmetric tensor space, so

\[
\|A\odot B\|_F
\le\|A\otimes B\|_F=\|A\|_F\|B\|_F.
\tag{S13.9}
\]

Thus the required operator constant obeys \(C_\odot\le1\) (indeed equality is attained on a common rank-one tensor in this normalization). Also

\[
\sup_\theta\|M_2(\theta)\|_F\le\sqrt q\,\Lambda;
\]

we therefore take the explicit envelope \(B_2^\star=\sqrt q\,\Lambda\). Since

\[
\|A^{\odot2}-B^{\odot2}\|_F
\le \|A-B\|_F(\|A\|_F+\|B\|_F),
\tag{S13.10}
\]

the cumulant estimator

\[
\widehat K_4=\widehat M_4-3\widehat M_2^{\odot2}
\tag{S13.11}
\]

obeys, on the intersection of (S13.7)--(S13.8),

\[
\|\widehat K_4-K_4\|_F
\le\epsilon_{K,N}
:=\epsilon_{4,N}^{\rm raw}
+3\epsilon_{2,N}(2B_2^\star+\epsilon_{2,N}).
\tag{S13.12}
\]

Thus the rectangular event (3.9) has probability at least \(1-\alpha_D\), uniformly over the frozen class. For fixed confidence and fixed \(q\), both radii are \(O(N^{-1/2})\) once the fixed block threshold is passed. Below

\[
N_{\rm MoM}=\max(B_2,B_4)
\tag{S13.13}
\]

the procedure returns the full compact training class and uses the shell radius \(C_{\rm shell}s\).

The rectangular interface permits a later sharper, still uniformly valid radius construction to replace (S13.7)--(S13.12). Such a replacement changes finite-sample width and operational behavior, but not the estimand, quotient-safe profile form, or conditional-coverage proof. Its diameter rate must be reverified from the replacement's certified radii.